\documentclass[utf8]{FrontiersinHarvard} 
\usepackage{url,hyperref,lineno,microtype,subcaption}
\usepackage[onehalfspacing]{setspace}
\def\keyFont{\fontsize{8}{11}\helveticabold }
\def\firstAuthorLast{} 
\def\Authors{Jun Inukai\,$^{1, \dagger}$, Tadahiro Taniguchi\,$^{1, \dagger,*}$, Akira Taniguchi\,$^{1}$, and Yoshinobu Hagiwara\,$^{1}$}
\def\Address{$^{1}$Ritsumeikan University \\ 1-1-1 Noji\ Higashi, Kusatsu, Shiga 525-8577, Japan.}
\def\corrAuthor{Tadahiro Taniguchi}
\def\corrEmail{taniguchi@em.ci.ritsumei.ac.jp}
\usepackage{amsmath,amsthm}
\usepackage{algorithm}
\usepackage{algpseudocode}

\theoremstyle{plain}
\newtheorem{thm}{Theorem}
\newtheorem*{thm*}{Theorem}
\theoremstyle{remark} 

\usepackage{diagbox}
\usepackage{enumerate}
\usepackage{multirow}
\usepackage{multibib}
\usepackage[normalem]{ulem}
\theoremstyle{plain}
\newtheorem{theorem}{Theorem}[section]
\newtheorem{corollary}[theorem]{Corollary}

\theoremstyle{definition}

\theoremstyle{remark}

\begin{document}
\onecolumn
\firstpage{1}
\title{Recursive Metropolis-Hastings Naming Game:\\
Symbol Emergence in a Multi-agent System based on Probabilistic Generative Models}
\author[\firstAuthorLast ]{\Authors } 

\address{} 
\correspondance{} 
\extraAuth{}
\extraAuth{$\dagger$ These authors contributed equally to this work and share the first authorship}
\maketitle
\begin{abstract}
In the studies on symbol emergence and emergent communication in a population of agents, a computational model was employed in which agents participate in various language games. Among these, the Metropolis-Hastings naming game (MHNG) possesses a notable mathematical property: symbol emergence through MHNG is proven to be a decentralized Bayesian inference of representations shared by the agents. However, the previously proposed MHNG is limited to a two-agent scenario. This paper extends MHNG to an $N$-agent scenario.
The main contributions of this paper are twofold: (1) we propose the recursive Metropolis-Hastings naming game (RMHNG) as an $N$-agent version of MHNG and demonstrate that RMHNG is an approximate Bayesian inference method for the posterior distribution over a latent variable shared by agents, similar to MHNG; and (2) we empirically evaluate the performance of RMHNG on synthetic and real image data, enabling multiple agents to develop and share a symbol system. Furthermore, we introduce two types of approximations—one-sample and limited-length—to reduce computational complexity while maintaining the ability to explain communication in a population of agents.
The experimental findings showcased the efficacy of RMHNG as a decentralized Bayesian inference for approximating the posterior distribution concerning latent variables, which are jointly shared among agents, akin to MHNG. Moreover, the utilization of RMHNG elucidated the agents' capacity to exchange symbols. Furthermore, the study discovered that even the computationally simplified version of RMHNG could enable symbols to emerge among the agents.
\tiny
\keyFont{\section{Keywords:} symbol emergence, emergent communication, probabilistic generative models, language game, Bayesian inference, multi-agent system}
\end{abstract}

\section{Introduction}
The origin of language remains one of the most intriguing mysteries of human evolution~\citep{deacon1997symbolic,christiansen2022language,Steels15}. Humans utilize various symbol systems, including gestures and traffic lights. Although the language is considered a type of symbol (or sign) system, it boasts the richest structure and the strongest ability to describe events among all symbol systems~\citep{Chandler2002}. The adaptive, dynamic, and emergent nature of symbol systems is a common feature in human society~\citep{taniguchi2018symbol,taniguchi2021semiotically}.
This paper focuses on the emergent nature of general symbols and their meanings, rather than the structural complexity of language.
The meaning of signs is determined within society in bottom-up and top-down manners, owing to the nature of symbol systems~\citep{taniguchi2016symbol}. More specifically, the self-organized (or emergent) symbol system enables each agent to communicate semiotically with others, while being subject to the top-down constraints of the emergent symbol system. Agent-invented symbols can hold meaning within a society of multiple agents, even though no agent can directly observe the intention in the brain of a speaker.  Peirce, the founder of semiotics, defines a symbol as a triadic relationship between sign, object, and interpretant~\citep{Chandler07}. The interpretant serves as a mediator between the sign and the object. In nature, the relationship between sign and object exhibits arbitrariness. This implies that human society—a multi-agent system using a symbol system—must form and maintain these relationships within a society in a decentralized manner. Research on symbol emergence, language evolution, and emergent communication has been addressing this issue using a constructive approach for decades. 

Studies on emergent communication takes many forms. Numerous studies have explored emergent communication by engaging agents in Lewis-style signaling games, such as referential games. ~\citet{Lazaridou17,lazaridou2018emergence} and \citet{havrylov2017emergence} demonstrated that agents can communicate using their own language by performing reference games. Furthermore, ~\citet{choi2018compositional} and \citet{mu2021emergent} suggested that the compositionality of emergent language can be improved by modifying the real image data used in reference games. However, \citet{bouchacourt2018agents} highlighted the issue of agents being able to communicate even when using uninterpretable images in referential games. \citet{noukhovitch2021} demonstrated the necessity of referential games for agent communication. Numerous studies have also attempted emergent communication with multiple agents. \citet{gupta2021dynamic} explored extending to multiple agents using meta-learning, while \citet{lin2021learning} employed autoencoder, a standard representation learning algorithm. \citet{chaabouni2022emergent} investigated the effects of varying the number of agents in referential games on agent communication. These studies successfully achieved communication through games that provided rewards.

In contrast, \citet{taniguchi2023emergent} proposed an alternative formulation of emergent communication based on probabilistic generative models and the assumption of joint attention. The Metropolis-Hastings naming game (MHNG) was introduced to explain the process by which two agents share the meaning of signs in a bottom-up manner from a Bayesian perspective. It was demonstrated that symbol emergence can be considered decentralized Bayesian inference. MHNG assumes {\it joint attention} between two agents—widely observed in human infants learning vocabularies—instead of reward feedback from a listener to a speaker. This idea is rooted in the concept of a symbol emergence system~\citep{taniguchi2016symbol,taniguchi2018symbol}, rather than the view of an emergent communication channel often assumed in emergent communication studies based on Lewis-style signaling game.
The notion of a symbol emergence system was proposed to capture the overall dynamics of symbol emergence from the perspective of emergent systems, i.e., complex systems exhibiting emergent properties. This approach aims to further investigate the fundamental cognitive mechanisms enabling humans to organize symbol systems within a society in a bottom-up manner. 

In this paper, we use the term {\it symbol system}
 in a restricted sense. Here, a symbol system simply refers to a set of signs and their (probabilistic) relationship to objects. In the context of studies on symbol emergence and emergent communication, we cannot assume a ground-truth relationship between signs and objects, unlike many studies in artificial intelligence, e.g., standard pattern recognition task that assumes a ground-truth label given to each object.    
Ideally, a multi-agent system should form a symbol system with which agents can appropriately categorize (or differentiate) objects and associate signs with objects. The definition of appropriate categorization and sign sharing is crucial to the formulation of symbol emergence. Different approaches assume different goals of symbol emergence and criteria based on various hypothetical principles. {\it Iterated learning} assumes that the goal of symbol emergence is for each agent to use the same sign for each object. In contrast, {\it emergent communication based on referential games} assumes that organizing signs allows a speaker to provide information that enables a listener to choose an object intended by the speaker. Taniguchi et al. proposed a collective predictive coding (CPC) hypothesis in the discussion of~\citet{taniguchi2023emergent}. The CPC hypothesis posits that the goal of symbol emergence is the formation of global representations created by agents in a decentralized manner. This can also be called {\it social representation learning}, i.e., symbol emergence is conducted as a representation learning process by a group of individuals in a decentralized manner. From a Bayesian perspective, this can be regarded as decentralized Bayesian inference.

\begin{figure}[tb!p]
    \centering
    \includegraphics[width=\linewidth]{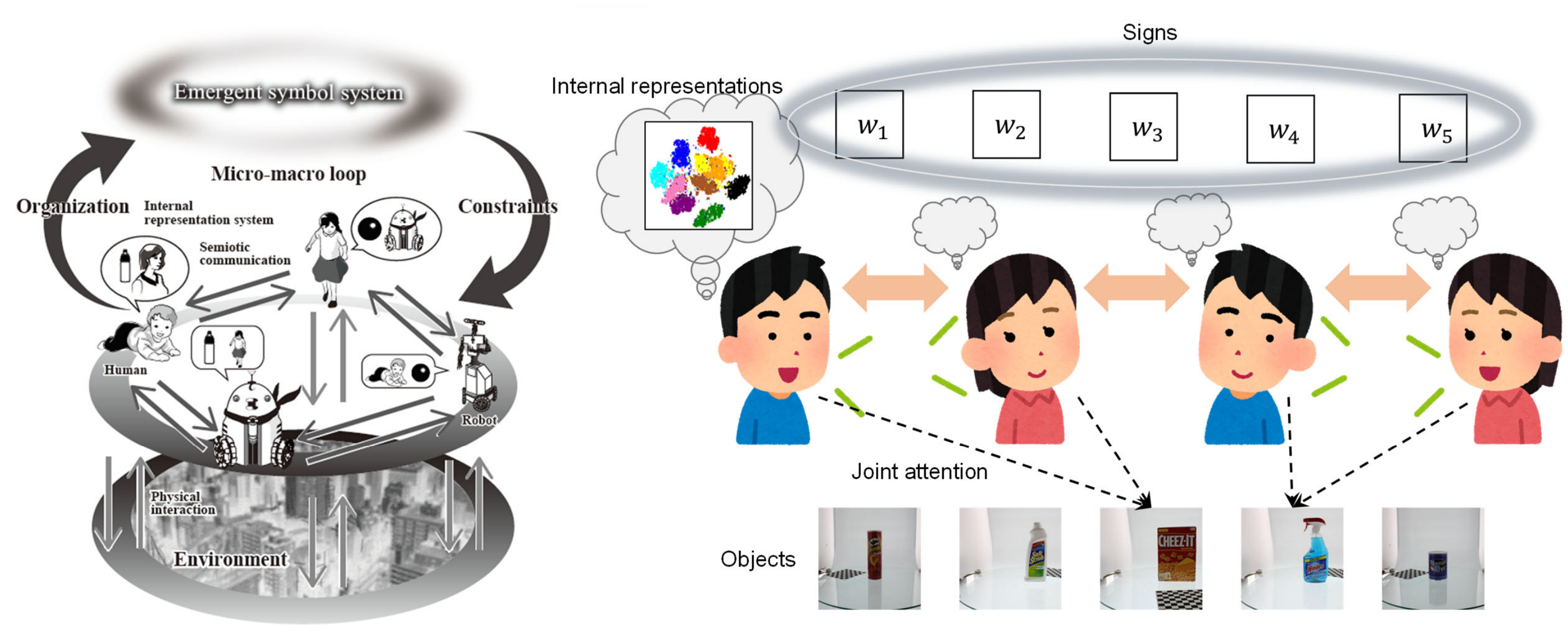}
    \caption{Left: an overview of symbol emergence system~\citep{taniguchi2016symbol}. Right: an overview of recursive Metropolis-Hastings naming game played among multiple agents.}
    \label{fig:overview}
\end{figure}

The MHNG was proposed, demonstrating that the language game enables two agents (agents 1 and 2) to form a symbol system, with MHNG's process mathematically considered as a Bayesian inference of $p(w\mid o^1, o^2)$, where $o^1$ and $o^2$ represent the observations of agents 1 and 2, and $w$ represents the shared representations, i.e., signs. 
Furthermore, MHNG does not assume the existence of explicit feedback from the listener to the speaker in the game, unlike Lewis-style signaling games widely employed in emergent communication studies. Instead, MHNG assumes joint attention, considered foundational to language acquisition during early developmental stages~\citep{cangelosi2015developmental}. The CPC hypothesis and MHNG are based on {\it generative models} rather than {\it discriminative models}, which are prevalent in the dominant approach to emergent communication in the deep learning community. The MHNG and results of constructive studies substantiate the CPC hypothesis in a tangible manner~\citep{taniguchi2023emergent}.
However, existing studies on MHNG only demonstrate that the game can become a decentralized approximate Bayesian inference procedure in a two-agent scenario. No theoretical research or evidence exists to show that the CPC hypothesis can hold in more general cases, i.e., in $N$-agent settings where $N \ge 3$. In other words, it is crucial to determine whether a language game can perform decentralized approximate Bayesian inference of $p(w\mid o^{1:N})$, where $o^{1:N} = \{o^1, o^2, \ldots,o^n,\ldots ,o^N \}$, and $o^n$ represents the observations of the $n$-th agent.

The fundamental reason why the MHNG can act as an approximate Bayesian inference of $p(w\mid o^1, o^2)$ is that the utterance of a sign $w \sim p(w\mid o^{Sp})$ by the speaker agent $Sp$ can be sampled based on agent $Sp$'s observations alone, and the acceptance ratio of the sign (i.e., the message) can be solely determined by the listener $Li$ based on its own observations and internal state. These properties are derived from the theory of the Metropolis-Hastings algorithm. MHNG has a solid theoretical basis in Markov Chain Monte Carlo (MCMC)~\citep{Hastings70}. 
However, the proof provided by \citet{taniguchi2023emergent} assumed that the naming game is played between only two agents. This assumption was based on the need for individual agents to make the proposal sampling of a sign and the acceptance/rejection decision, respectively, without direct observation of the internal states of the other agent.
Due to the difficulty, a naming game having the same theoretical property as MHNG for the $N$-agent ($N \ge 3$) case has not been proposed.

The goal of this paper is to extend the MHNG to the $N$-agent ($N \ge 3$) scenario and show that the extended naming game can act as an approximate Bayesian inference algorithm for $p(w\mid o^{1:N})$. The main idea of the proposed method is the introduction of a{\it recursive structure} into the MHNG. Let us consider a $3$-agent case. If $w \sim p(w\mid o^1, o^2)$ can be sampled in the MH algorithm, the acceptance ratio for the third agent can be calculated based on the third agent's internal states, and the communication can be regarded as a sampling process of $p(w\mid o^1, o^2, o^3)$. Notably, $w \sim p(w\mid o^1, o^2)$ can be sampled using the original two-agent MHNG. By extending this idea in a recursive manner, we can develop a recursive MHNG (RMHNG). The details will be described in Section~\ref{sec:MH-NG}.

The main contributions of this paper are twofold. 
\begin{itemize}    
\item We propose the RMHNG played between $N$ agents and provide mathematical proof that the RMHNG acts as an approximate Bayesian inference method for the posterior distribution over a latent variable shared by the agents given the observations of all the agents.    \item The performance of the RMHNG is empirically demonstrated on synthetic data and real image data. The experiment shows that the RMHNG enables more than two agents to form and share a symbol system. The inferred distributions of signs are shown to be a posterior distribution over $p(w\mid o^{1:N})$ in an empirical manner. To reduce computational complexity and maintain applicability for the explanation of communication in human society, two types of approximations, i.e., (1) one-sample (OS) approximation and (2) limited-length (LL) approximation, are proposed and both are validated through experiment. 
\end{itemize}
The remainder of this paper is structured as follows. In section~\ref{sec:MH-NG}, we describe RMHNG, explaining its assumed generative model, algorithms, and theoretical results. Additionally, a practical approximation is provided. Section 3 presents an experiment using synthetic data and demonstrates the RMHNG empirically. Section 4 presents an experiment using the YCB object dataset~\citep{calli2015ycb}, which contains real images of everyday objects. In Section 5, we engage in a comprehensive discussion. Finally, we conclude the paper in Section 6.

\section{Recursive Metropolis-Hastings naming game (RMHNG)}\label{sec:MH-NG}
\subsection{Overview}
The RMHNG is a language game played between multiple agents ($N \ge 2$). It is an extension of the original MHNG. When $N=2$, the RMHNG is equivalent to the original MHNG. Notably, the game does not allow agents to give any feedback to other agents during the game, unlike Lewis-style signaling games~\citep{lewis2008convention}, which have been used in studies of emergent communication. Instead, the game assumes joint attention.
Generally, when we ignore the representation learning parts, the original MHNG is played as follows:
\begin{enumerate}    
\item For each object $d \in \mathbb{D}$, the $n$-th agent (where $n\in \{1, 2\}$) views the $d$-th object and infers the internal state $x_d^n$, its percept, from its observations $o_d^n$, i.e., calculate $p(x_d^n\mid o_d^n)$ or sample $x_d^n \sim p(x_d^n\mid o_d^n)$, where $\mathbb{D}$ is a set of object. Set the initial roles to $\{Sp, Li\} = \{1, 2\}$. 
\item The $Sp$-th agent 
says a sign $w_d^{Sp}$ (i.e., a word) corresponding to the $d$-th object in a probabilistic manner by sampling a word from the posterior distribution over words (i.e., $w_d^{Sp} \sim p(w_d\mid x_d^{Sp})$) for each $d$.    
\item Let a counterpart, that is, a listener, be $Li$-th agent. 
Assuming that the listener is looking at the same object, i.e., joint attention, the listener determines whether to accept the word based on its belief state with probability $r = {\rm min}\left(1, 
    \frac{
    P(x^{Li}_{d}\mid \theta^{Li},w_d^{Sp})}
    {
    P(x^{Li}_{d}\mid \theta^{Li}, w^{Li}_{d})         
    }
    \right)$.  A listener updates its internal parameter $\theta^{Li}$.
\item They alternate their roles, i.e., take turns, and go back to 2.
\end{enumerate}
The RMHNG extends the original MHNG to allow for communication between multiple agents ($N \ge 3$) and forms a shared symbol system among them. The key idea of RMHNG is as follows:
\begin{enumerate}    
\item In an RMHNG played by $M$ agents, we recursively use an RMHNG played by $M-1$ agents as a proposal distribution of $w_d$, which corresponds to a speaker in the original MHNG.    Note that, an RMHNG played by $M-1$ agents $(1, \ldots ,M-1)$ is a sampler of an approximate distribution of $p(w_d\mid x_{1:M-1})$.    
\item An RMHNG played by two agents ($N=2$) is equivalent to an original MHNG.
\end{enumerate}
Consequently, when played by $N$ agents, the RMHNG approximates the distribution of $p(w_d\mid x_{1:N})$ through mathematical induction.

\subsection{Generative model}

 \begin{table}[btp]
     \centering
     \begin{tabular}{cc} \hline
     Variable & Explanation  \\\hline
       $w_d$   &  A sign, e.g., a name, for the $d$-th object.\\
       $x^*_d$   & Perceptual state or feature vector corresponding to the $d$-th object. \\
       $o^*_d$   & Observation for the $d$-th object   \\
       $\theta^*$   &  Knowledge about the relations between signs and perceptual states.  \\
       $\phi^*$   & Knowledge about relations between perceptual states and observations. \\
       $\alpha$   &  A hyperparameter for  $\theta^*$ \\
       $\beta$   &   A hyperparameter for  $\phi^*$       \\\hline
     \end{tabular}
     \caption{Variables of Inter-PGM and their explanations. Superscript $* \in \mathbb{N} $ refers to a specific agent.}
     \label{tab:inter-pgm} 
 \end{table}

\begin{figure}[tb!p]
    \centering
    \includegraphics[width=1.0\linewidth]{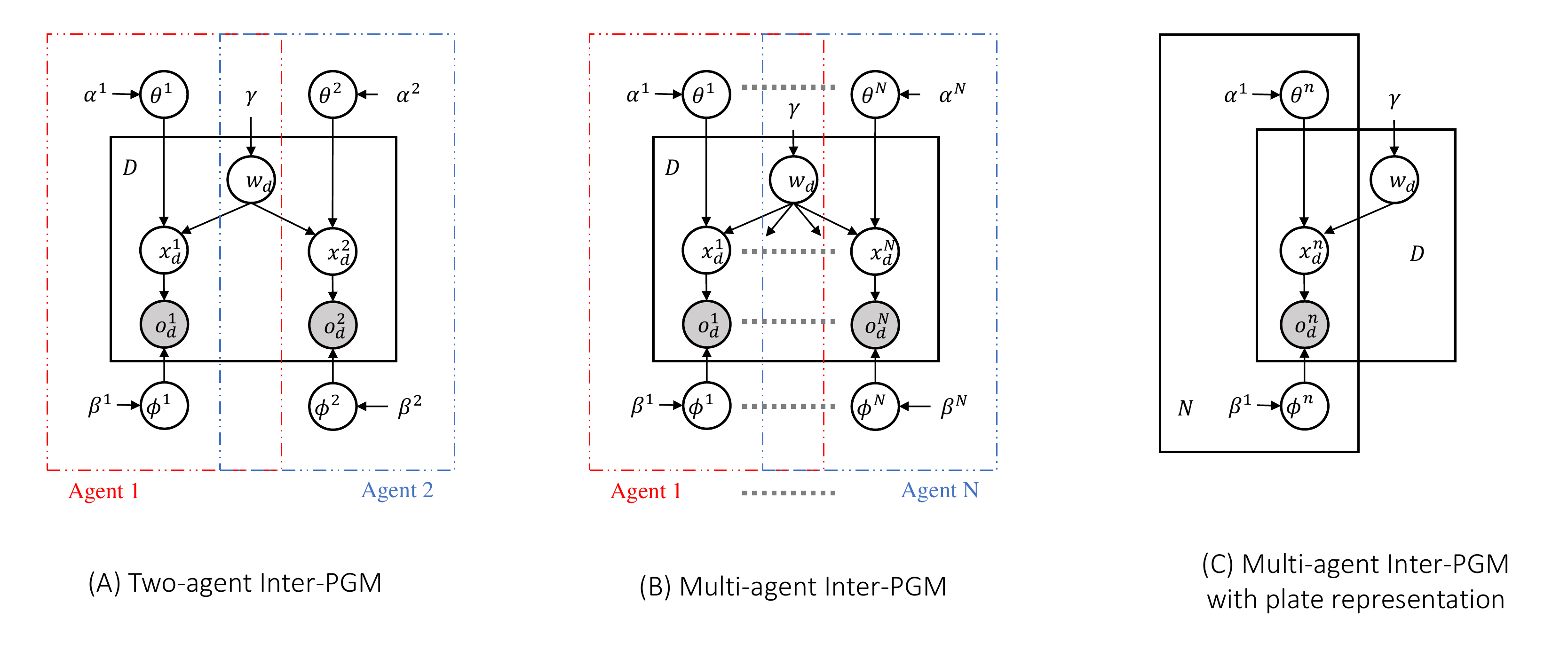}
    \caption{Probabilistic graphical models considered for MHNG and RMHMG.
    (A) PGM is for MHNG, i.e., a two-agent scenario called two-agent Inter-PGM. (B) PGM is a generalization of PGM in (A), i.e., a multi-agent scenario ($N \ge 2$), called multi-agent Inter-PGM. $n$-th agent has variables for observations $o^n_d$, internal representations $w^n_d$ for the $d$-th object ($1 \le d \le D$). $n$-th agent has global parameters $\phi^n$ and $\theta^n$ and hyperparameters. Variable $w_d$ is a shared latent variable, and concrete samples drawn from the posterior distribution over $w_d$ are regarded as an utterance, i.e., a sign. (C) PGM shows a concise representation of (B) using plate representations (i.e., (B) and (C) represent the same probabilistic generative process).}
    \label{fig:pgms}
\end{figure}

Figure~\ref{fig:pgms} presents three probabilistic graphical models (PGMs) representing the interactions between multiple agents sharing a latent variable $w_d$.
(A) The left panel shows a PGM that integrates two PGMs representing two agents with a shared latent variable $w_d$. This model is referred to as the two-agent Inter-PGM.
(B) The center panel generalizes the PGM in (A) to integrate PGMs representing $N$ agents. This model can be considered a multimodal PGM in which a shared latent variable integrates multimodal observations. We refer to this model as the multi-agent Inter-PGM.
(C) The right panel provides a concise representation of (B) using plate representations, meaning (B) and (C) represent the same probabilistic generative process.
When agent $n$ observes the $d$-th object, they receive observations $x_d^n$ and infer their internal representation $x_d^n$. A latent variable representing a word, $w_d$, is shared among the agents. 
The inference of $\theta^n$ and $x^n_d$ corresponds to a general representation learning problem. As studied in ~\citet{taniguchi2023emergent}, 
introducing the Symbol Emergence in Robotics Toolkit (SERKET) framework~\citep{nakamura2018serket,taniguchi2020neuro} allows us to decompose the main part of the naming game (exchanging signs $w^n_d$ between agents) and the representation learning part (inferencing $x^n_d$ and $\theta^n$).
For simplicity and to focus on the extension of the MHNG, we assume that $x_d^n$ is observable throughout the paper and concentrate on inferring $\theta^n$ and $w_d$ through the language game.


\subsection{Inference as a naming game}
The RMHNG, like the MHNG, acts as a decentralized approximate Bayesian inference based on the MH algorithm. A standard inference scheme for $p(w_d\mid x_d^{1:N})$ in Figure~\ref{fig:pgms} (C) requires the information about $x_d^{1:N}$, e.g., the posterior distribution $p(x_d\mid o_d^{1:N})$. However, $x_d^{1:N}$ are internal representations of each agent, and the agents cannot access each other's internal state, which is a fundamental principle of human semiotic communication.
If the agents' brains were connected, the shared variable $w_d$ would be a representation of the connected brain and could be inferred by referencing $x^{1:N}_d$. But this is not the case in real-world communication. The challenge is to infer the shared variable $w_d$ without connecting the agents' brains and without simultaneously referencing $x^{1:N}_d$. The solution is to play the RMHNG.

\begin{figure}[tb!p]
    \centering
    \includegraphics[width=1.0\linewidth]{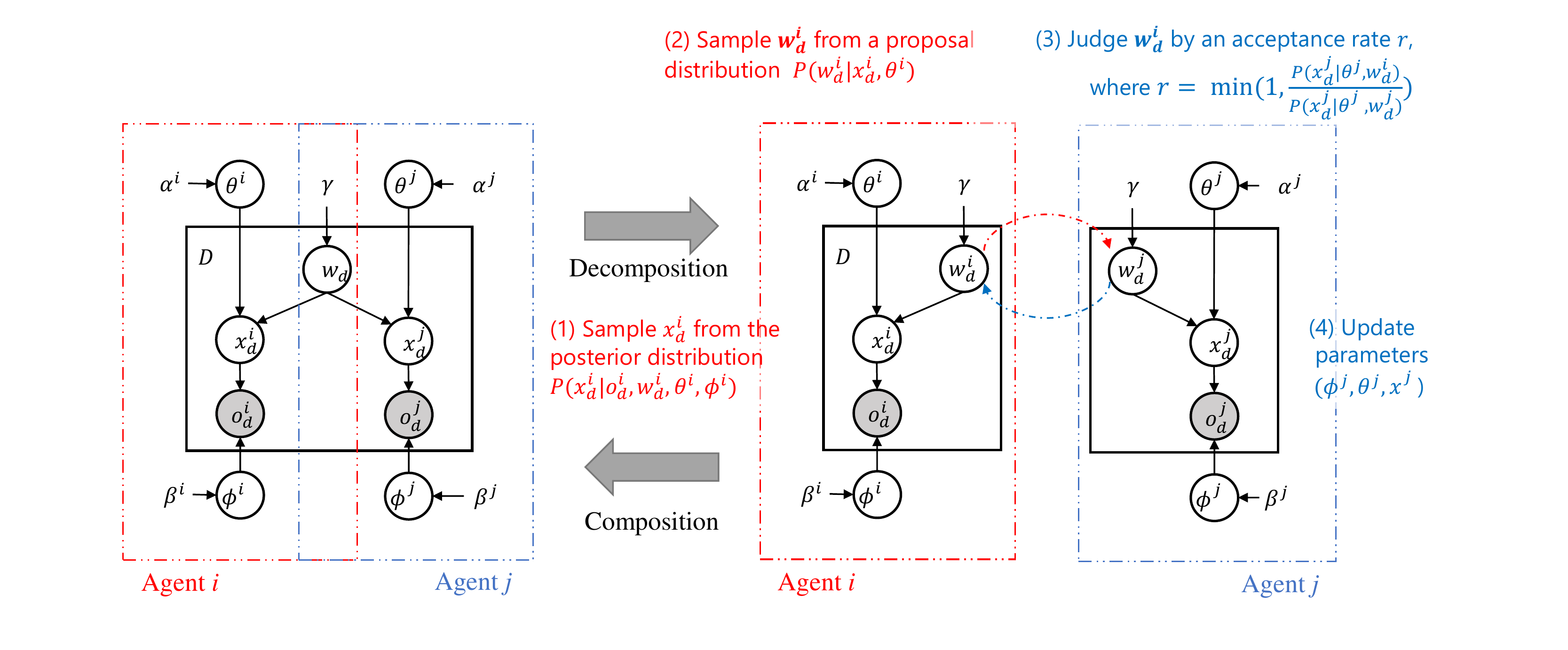}
    \caption{Decomposition and composition of two-agent Inter-PGM. Notes (1) -- (4) describe the MH communication(Algorithm~\ref{alg:MH-C}), which is an elemental step of MHNG. Similarly, $N$-agent multi-agent Inter-PGM can be decomposed into $N$ PGMs representing $N$ agents.}
    \label{fig:decomposition}
\end{figure}


The decomposition of the generative model inspired by SERKET, as shown in Figure~\ref{fig:decomposition} right, allows for a more manageable and systematic approach to the inference of hidden variables. The SERKET framework enables the decomposition of a PGM into multiple modules, which simplifies the overall inference process by breaking it down into inter-module communication and intra-module inference~\citep{taniguchi2020neuro,nakamura2018serket}.
In the context of the RMHNG, the semiotic communication between agents is analogous to the inter-module communication in the SERKET framework.

\subsubsection{MH receiving}
Algorithm~\ref{alg:MH-P} presents the MH-receiving algorithm. When a listener agent $A^{Li} \in \mathbb{A}$ receives a sign $w^\star$ for the $d$-th object, the agent evaluates whether to accept the sign and update $A^{Li}.w_d$ or not, where $\mathbb{A}$ is a set of agents. Here, $A^{i}.w_d$ represents the $w_d$ that agent $A^{i}$ possesses. Similarly, $A^{Li}.x_d$ denotes the $x_d$ held by agent $A^{i}$. For $A^n \in \mathbb{A}$, $A^n$ is an instance of a struct (or a class), and $A^n.\bullet$ indicates the variable $\bullet$ of the $n$-th agent, i.e., $(A^n.\theta, A^n.\phi, (A^n.o_d)_{d \in \mathbb{D}}, (A^n.x_d)_{d \in \mathbb{D}}, (A^n.w_d)_{d \in \mathbb{D}})  =(\theta^n, \phi^n, (o_d^n)_{d \in \mathbb{D}}, (x_d^n)_{d \in \mathbb{D}}, (w_d^n)_{d \in \mathbb{D}})$.
The function {\textbf MH-receiving} returns the sign for the $d$-th object agent $Li$ holds after receiving a new name for the $d$-th object from another agent. 

\begin{algorithm}[h]
\caption{MH Receiving}
\label{alg:MH-P}
\begin{algorithmic}[1]
\Function{MH-RECEIVING}{$w^\star,A^{Li}, d$}
\State $r = {\rm min}\left(1,
    \frac{
    P(A^{Li}.x_{d}\mid A^{Li}.\theta,w^\star)}
    {
    P(A^{Li}.x_{d}\mid A^{Li}.\theta,A^{Li}.s_{d})         
    }
    \right)$
\State $u \sim {\rm Unif}(0,1)$
\If {$u\leq r$}
\State {\bf return} $w^\star$
\Else
\State {\bf return} $A^{Li}.w_d$
\EndIf
\EndFunction
\end{algorithmic} 
\end{algorithm}

\subsubsection{MH communication}
Algorithm~\ref{alg:MH-C} presents the MH-communication algorithm. The function \textbf{MH-communication} describes the elementary communication in both the MHNG and the RMHNG. A sign $s$ for the $d$-th object is sampled (i.e., uttered) by agent $Sp$ and received by agent $Li$, where $Li, Sp \in \mathbb{N}$.

\begin{algorithm}[h]
\caption{MH Communication}
\label{alg:MH-C}
\begin{algorithmic}[1]
\Function{MH-communication}{$A^{Sp},A^{Li},d$}
\State $w^\star \sim P(A^{Sp}.w_{d}\mid A^{Sp}.x_{d},A^{Sp}.\theta)$
\State {\bf return} \textrm{MH-receiving}$(w^\star, A^{Li}, d)$
\EndFunction
\end{algorithmic} 
\end{algorithm}

\subsubsection{Recursive MH communication}
Algorithm~\ref{alg:MH-PB} presents the recursive MH communication algorithm. This algorithm represents the recursive MH communication process, as shown in Figure~\ref{fig:rmhc}. The recursive MH communication is one of the MH sampling procedures for $p(w_d\mid o^{1:N}_d)$. Given $n+1$ ($n<N$) agents, each with parameter $w_d$, this algorithm is used to compute $w_d$ for interactions among $n$ agents.
If $n>1$, the RMH-communication function is recursively called for agents $A^{1:n-1} \subset \mathbb{A}$ to compute interactions among them. Then, $A^{n+1}$ updates its own parameter $w_d$ using the received information $\bar{s}$ by calling the MH-receiving function. 
If $n=1$, the MH-communication function is called.
After the internal loop (from line 2 to line 9) is completed, the algorithm returns the $w_d$ of a randomly selected agent $j$ from $A^{1:n+1}$. This algorithm can recursively calculate interactions among $N$ agents.

\begin{algorithm}[h]
\caption{Recursive Metropolis-Hastings Communication}
\label{alg:MH-PB}
\begin{algorithmic}[1]
\Function{RMH-communication}{$A^{1:n}, A^{n+1}, d$}
\For{$t=1$ to $T$} 
\Comment{Internal iteration} 
\If {$n>1$}
\State $\bar{w} \leftarrow$ RMH-communication($A^{1:n-1}, A^{n}$, d)
\State $A^{n+1}.w_d \leftarrow \textrm{MH-receiving}(A^{n+1}, \bar{w}, d)$
\Else
\State $A^{2}.w_d\leftarrow\textrm{MH-communication}(A^1,A^2,d)$
\EndIf
\EndFor
\State $j \leftarrow $rand($\{1:n+1\}$) 
\State \Return $A^j.w_d$
\EndFunction
\end{algorithmic} 
\end{algorithm}

\begin{figure}[tb!p]
    \centering
    \includegraphics[width=1.0\linewidth]{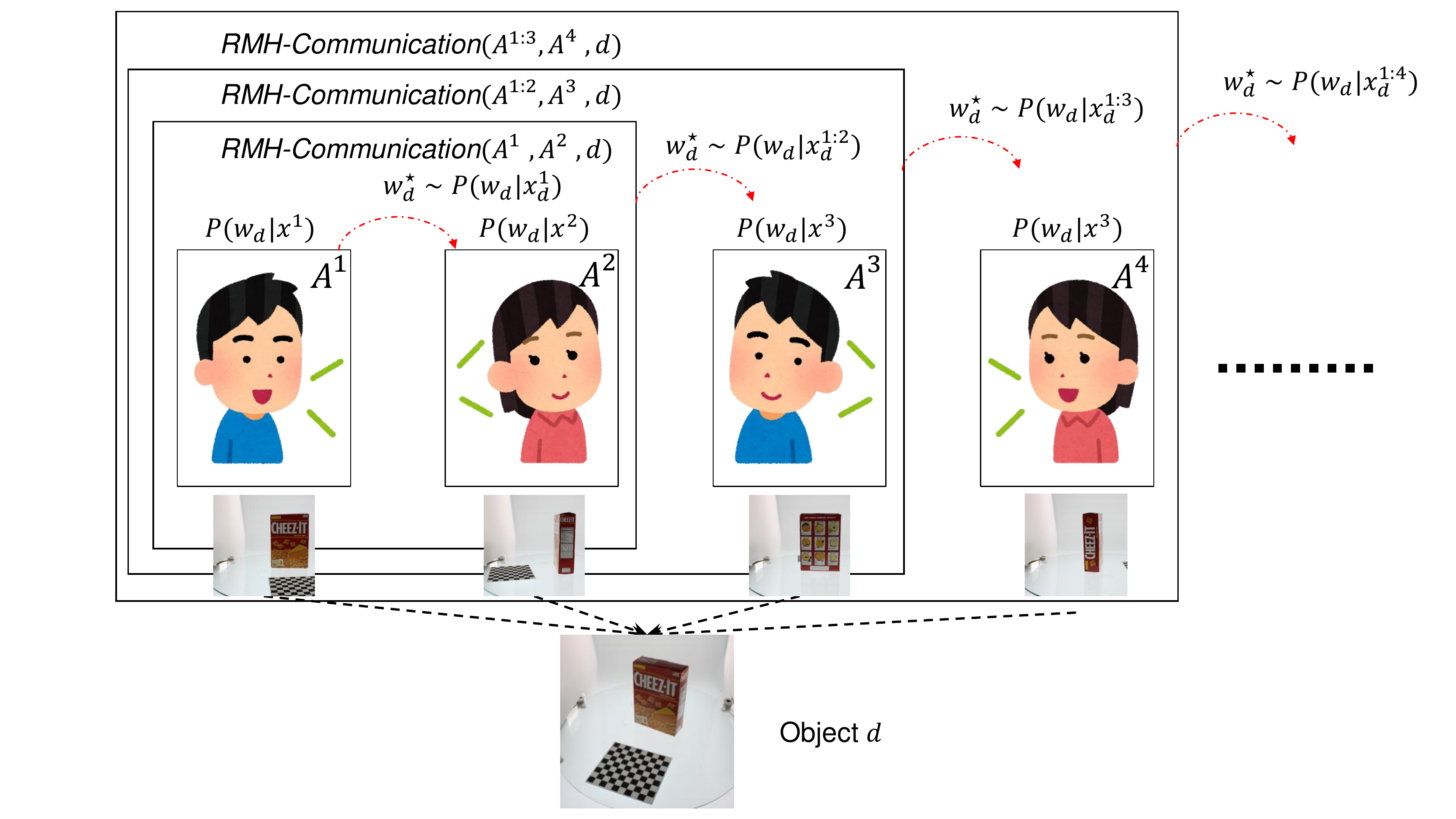}
    \caption{The upper figure is schematic explanation of RMH communication and RMHNG. The recursive MH communication is one of the MH sampling procedures for $p(w_d\mid o^{1:N}_d)$. Given $n+1$ ($n<N$) agents, each with parameter $w_d$, this algorithm is used to compute $w_d$ for interactions among $n$ agents.}
    \label{fig:rmhc}
\end{figure}

\begin{figure}[tb!p]
    \centering
    \includegraphics[width=1.0\linewidth]{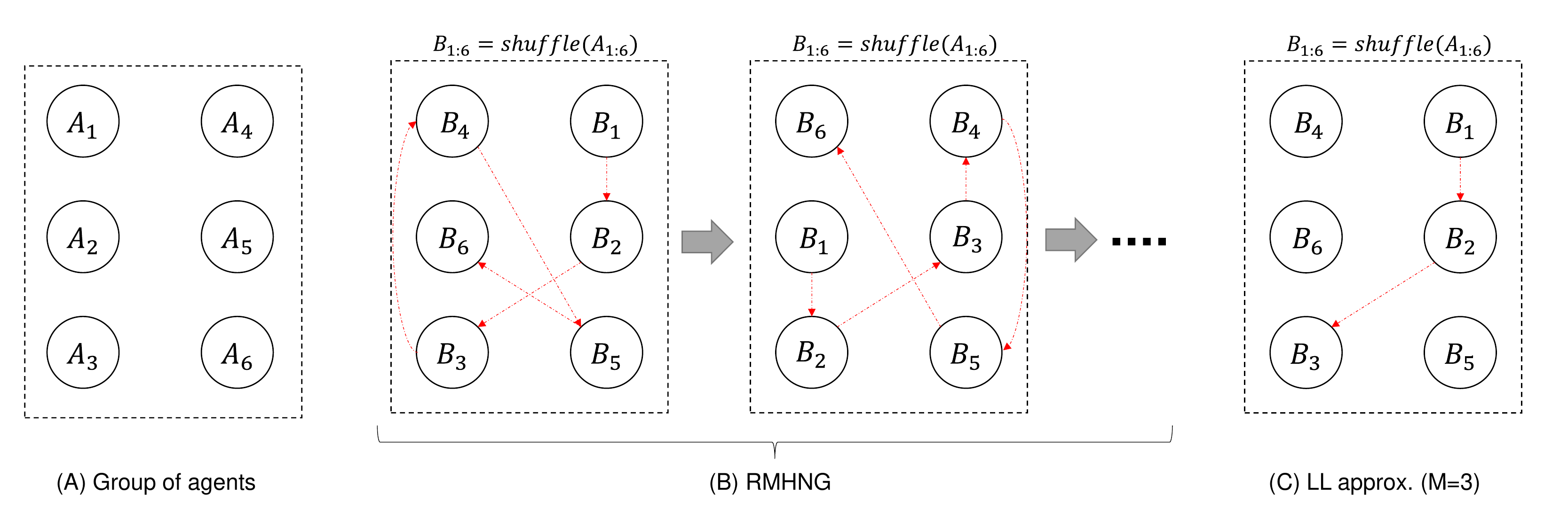}
    \caption{Schematic explanation of flow of RMHNG (see Algorithm~\ref{alg:RMHNG}) and limited-length approximation (in case where $M=3$) (see Section~\ref{sec:approximations}).}
    \label{fig:chain}
\end{figure}

\subsubsection{Recursive MH naming game}
Algorithm~\ref{alg:RMHNG} presents the recursive MH naming game algorithm. The agents repeatedly engage in recursive MH communication for each object, shuffling the order of the agents. The recursive MH communication is mathematically a type of approximate MH sampling procedure for $p(w_d\mid o^{1:N}_d)$.
After the recursive MH communication is performed for every object, each agent internally updates its global parameter $\theta^n$. 
By iterating this block $I$ times, the agents can sample $\{w_d\}, \{\theta_n\}$ from the posterior distribution over $p(\{w_d\}, \{\theta_n\}\mid \{x^n_d\}) $.

\begin{algorithm}[h]
\caption{Recursive Metropolis-Hastings naming game}
\label{alg:RMHNG}
\begin{algorithmic}[1]
\State Explanation of each parameter
\State $A^{1:N}= (A^1,A^2,\ldots, A^N)$ 
\State {Initialize all parameters}
\For{$i=1$ to $I$} 
\Comment{Number of iteration}
\State  $B^{1:N}=\mathrm{shuffle(A^{1:N})}$ \Comment{Randomization of the order of the agents}
\For{$d=1$ to $D$} 
\Comment{Naming every object}
\State RMH-communication($B^{1:M-1}, B^{M}, d$)
\Comment{$M=N$ : No approximation}
\State 
\Comment{$M<N$ : Limited-length approximation (see \ref{sec:approximations})}
\EndFor
\For{$n=1$ to $N$} 
\Comment{Parameter update}
\State $B^{n}.\theta \sim P(B^{n}.\theta\mid B^{n}.s, B^{n}.x)$
\EndFor
\EndFor
\end{algorithmic} 
\end{algorithm}

\subsection{Theory and proof}

For the main theoretical result, we use the following corollary.
\begin{corollary} The MH communication is a Metropolis-Hastings sampler of $P(w_d\mid x_d^{Sp}, x_d^{Li}, \theta^{Sp}, \theta^{Li})$. \end{corollary} 
The acceptance probability $r$ in {\bf MH-receiving} is equivalent to that in the MH algorithm for $P(w_d\mid x_d^{Sp}, x_d^{Li}, \theta^{Sp}, \theta^{Li})$ in the case that $P(w\mid x^{Sp}, \theta^{Sp})$ is a proposal distribution. This result is a generalization of \citep{hagiwara2019symbol,YoshinobuHagiwara2022} and a special case of \citep{taniguchi2023emergent}. 
For the details of the proof, please refer to the original papers. 
   
The first theoretical result is as follows.
\begin{thm} The RMH communication converges to a MCMC sampler of $P(w_d\mid x_d^{1:n},\theta^{1:n})$ when $T \rightarrow \infty$. \end{thm} \begin{proof} When $n=2$, the RMH communication is reduced to the execution of MH communication $T$ times. The MH communication is proven to be an MH sampler in corollary 1. Therefore, RMH communication is a MCMC sampler, and the sample distribution converges to $P(w_d\mid x_d^{1:2},\theta^{1:2})$ when $T \rightarrow \infty$.
When $n>2$, if the $\mathrm{RMH-communication}(B^{1:n-1}, B^{n}, d)$ is a sampler for $P(w_d\mid x_d^{1:n},\theta^{1:n})$, $RMH-communication(B^{1:n}, B^{n+1}, d)$ becomes an MH sampler for $P(w_d\mid x_d^{1:n+1},\theta^{1:n+1})$. Therefore, RMH communication is a MCMC sampler, and the sample distribution converges to 
$P(w_d\mid x_d^{1:n+1},\theta^{1:N+1})$ when $T \rightarrow \infty$.
Therefore, the RMH communication converges to a MCMC sampler of $P(w_d\mid x_d^{1:n},\theta^{1:n})$ when $T \rightarrow \infty$ by mathematical induction.
\end{proof}

\begin{thm} The RMHNG converges to a MCMC sampler of $P(w_d, \theta^{1:n}\mid x_d^{1:n})$ when $T \rightarrow \infty$. \end{thm} \begin{proof} The RMHNG samples the local parameters $w_d$ for all $d$ using the RMH communication, and the global parameters $\theta^{1:n}$ from $P(\theta^{1:n}\mid \{x_d^{1:n}\}_{d \in \mathbb{D}} , \{w_d\}_{d \in \mathbb{D}})$. 
 When $T \rightarrow \infty$, RMH communication converges to a sampler of $P(w_d\mid x_d^{1:n},\theta^{1:n})$. As a result, the RMHNG converges to a Gibbs sampler of $P(w_d, \theta^{1:n}\mid x_d^{1:n})$. 
\end{proof}
As a result, the RMHNG is proved to be a decentralized approximate Bayesian inference procedure for $p(\{w_d\}_{d \in \mathbb{D}}, \{\theta_n\}_{n \in \{1, \ldots, N\}}\mid \{x^n_d\}_{d \in \mathbb{D}})$.

\subsection{Approximations}\label{sec:approximations}
Though the RMHNG is guaranteed to be a decentralized approximate Bayesian inference procedure for $p(\{w_d\}_{d \in \mathbb{D}}, \{\theta_n\}_{n \in \{1, \ldots, N\}}\mid \{x^n_d\}_{d \in \mathbb{D}})$
, the computational cost increases exponentially with respect to the number of agents $N$. The computational cost is $O(IDT^{(N-1)})$. This indicates that the computational cost of \textbf{RMH-communication}, i.e., $O(T^{(N-1)})$, has a significant impact on the overall computational cost. Therefore, we introduce a lazy version of RMHNG, which employs two approximations to reduce the computational cost.

\subsubsection{One-sample (OS) approximation}
The number of internal iterations $T$ corresponds to the iterations of MCMC for sampling $w_d$ given variables of a (sub)group of agents. Theoretically, $T$ should be large. However, practically, even $T=1$ can work in an approximate manner. We refer to the RMHNG with $T=1$ as the OS approximation (\textbf{OS}), a special case. With the OS, the computational cost of RMH communication is significantly reduced from $O(T^{(N-1)})$ to $O(N)$.

\subsubsection{Limited-length (LL) approximation }
RMH communication is a process of information propagation through a chain connecting $N$ agents (as shown in Figure~\ref{fig:chain}). Limited-length approximation (\textbf{LL}) truncates the chain to $M$ agents. By shuffling the order of the agents according to the data points, it is expected that sufficient information will be statistically propagated among all the agents. \textbf{LL} reduces the computational cost of RMH communication from $O(T^{(N-1)})$ to $O(T^{(M-1)})$, where $M \le N$ is the length of the truncated chain, i.e., the number of agents participating in an RMH communication.
To reduce computational complexity while maintaining applicability for explaining communication in human society, two types of approximations are proposed: (1) OS approximation and (2) LL approximation. Both types were validated through experimentation.

\subsection{Example: multi-agent Inter-GMM}\label{sec:inter-GMM}
To evaluate the RMHNG, we developed a computational model of symbol emergence called multi-agent Inter-GMM. This is based on the Gaussian mixture model (GMM) and is a special case of the multi-agent Inter-PGM. \citet{hagiwara2019symbol,YoshinobuHagiwara2022} proposed the Inter-Dirichlet mixture (Inter-DM) which combines two Dirichlet mixtures (DMs), $p(x^n_d\mid w_d)$ and $p(o^n_d\mid x^n_d)$, represented as categorical distributions in Figure~\ref{fig:pgms} (A). \cite{taniguchi2023emergent} proposed Inter-GMM+VAE which combines two GMM+VAEs, i.e., $p(x^n_d\mid w_d)$ and $p(o^n_d\mid x^n_d)$ represented as a categorical distribution as a part of GMM and a VAE respectively. Inter-GMM is defined as a part of Inter-GMM+VAE and combines two GMMs via a shared latent variable. We generalized the two-agent Inter-GMM and obtained the multi-agent Inter-GMM, which has $N$ Gaussian emission distributions corresponding to $N$ agents.
The probabilistic generative process of the multi-agent inter-GMM is as follows: \begin{align} w_d &\sim {\rm Cat}(\gamma) \quad &d = 1, \ldots, D\label{eq:inter-GMM+VAE-1}\\
\mu_k^n, \Lambda^n_k &\sim \mathcal{N}(\mu_k^n\mid m,(\bar{\alpha}\Lambda^n_k)^{-1}){\mathcal W}(\Lambda_k^n\mid \nu,\bar{\beta})\quad &k=1, \ldots, K\\
\alpha^n &= (m, \bar{\alpha},\nu, \bar{\beta}) \\
\theta^n &= (\mu_{1:K}^n, \Lambda^n_{1:K})\\
x^n_d &\sim \mathcal{N}(x^n_d\mid \mu^n_{w_d},(\Lambda^n_{w_d})^{-1})\quad &d = 1, \ldots, D\label{eq:inter-GMM+VAE-3}
\end{align} where $\mu_k^n$ and $\Lambda^n_k$ are the mean vector and the precision matrix of the $k$-th Gaussian distribution of the $n$-th agent.
${\rm Cat}(*)$ is the categorical distribution, $\mathcal{N}(*)$ is the Gaussian distribution, ${\mathcal W}(*)$ is the Wishart distribution.
The Inter-GMM is a probabilistic generative model represented by the PGM shown in Figure~\ref{fig:pgms} (C). In other words, the multi-agent Inter-GMM is an instance of the multi-agent Inter-PGM.
Therefore, the RMHNG can be directly applied to the multi-agent Inter-GMM. 

\section{Experiment 1: Synthetic data}\label{sec:syn}
\subsection{Conditions}
We evaluated the RMHNG using the multi-agent Inter-GMM with four agents ($N=4$) using synthetic data. For all experiments (excluding the measurement of computation time), the number of iterations ($I$) was set to $100$, and each experiment was conducted five times. 

{\bf Dataset}:  We created synthetic data to serve as observations for the four agents. A dataset was generated from five $4$-dimensional Gaussian distributions with mean vectors of $(0, 1, 2, 3)$, $(0, 5, 6, 7)$, $(8, 5, 10, 11)$, $(12, 13, 10, 15)$, and $(16, 17, 18, 15)$, respectively. The variance of each Gaussian distribution was set to the identity matrix $\mathbf{I}$. The values obtained for each dimension were taken as observations for each agent. In other words, the value of the $n$-th dimension of data sampled from the GMM was considered as the observation for the $n$-th agent. Notably, for the $n$-th agent, the $n$-th and $n+1$-th Gaussian distributions have the same mean and variance. Therefore, the $n$-th agent cannot differentiate the $n$-th and $n+1$-th Gaussian distributions without communication.

{\bf Compared methods}: We assessed the proposed model, {\it RMHNG} (proposal), by comparing it with two baseline models and a topline model. In {\it No communication} (baseline 1), two agents independently infer a sign $w$, i.e., perform clustering of the data. No communication occurs between the four agents. In other words, the {\it No communication} model assumes that the agents independently infer signs $w^n_d$ $(n \in \{1, 2, 3, 4\})$, respectively, using four GMMs. {\it All acceptance} (baseline 2) is the same as the RMHNG, with an acceptance ratio always set to $r = 1$ in MH receiving (see Algorithm~\ref{alg:MH-P}). Each agent always believes that the sign of the other is correct. In {\it Gibbs sampling} (topline), the sign $w_d$ is sampled using the Gibbs sampler. This process directly uses $x^{1:4}_d$, although no one can simultaneously examine the internal (i.e., brain) states of human communication. This is a centralized inference procedure and acts as a topline in this experiment.

We also evaluated two approximation methods introduced in Section~\ref{sec:approximations}. {\bf OS} and {\bf LL} refer to the OS and LL approximations, respectively.  
In the LL approximation, $M=2$, i.e., the chain length is one. In {\bf OS\&LL}, both {\bf OS} and {\bf LL} approximations were applied simultaneously.

{\bf Hyperparameters}: In all methods, the hyperparameters of the agents were set to be the same. The hyperparameters were $\beta=1$, $m=0$, $W=0.01$, and $\nu=1$.

{\bf Evaluation criteria}:
\begin{itemize} \item \textbf{Clustering:} We used Adjusted Rand Index (ARI)~\citep{hubert1985comparing} to evaluate the unsupervised categorization performance of each agent in the MH naming game. A high ARI value indicates excellent categorization performance, while a low ARI value indicates poor performance. ARI is advantageous over precision since it accounts for label-switching effects in clustering by comparing the estimated labels and ground-truth labels. Appendix~\ref{apdx:ari} provides more details. 
\item \textbf{Sharing sign:} We assessed the degree to which the two agents shared signs using the $\kappa$ coefficient ($\kappa$)~\citep{cohen1960coefficient}. Appendix~\ref{apdx:ari} provides more details. \item\textbf{Computation time:} We conducted experiments to measure the processing time of the program when running it at $I=10$ by varying the values of $T$ in Algorithm~\ref{alg:MH-PB} and $M$ in Algorithm~\ref{alg:RMHNG}. We conducted experiments with $T=1,2,3,4$ and $M=1,2,3$. The program was run three times in each experiment (30 iterations in total, initialized every 10 iterations), and we calculated the average processing time per iteration (10 iterations). \item\textbf{Decentralized posterior inference:} To investigate whether RMHNG is an approximate Bayesian estimator of the posterior distribution $p(w\mid x^1,x^2,\ldots,x^N,\theta^1,\theta^2,\ldots,\theta^N)$, we need to compare it with the true posterior distribution. However, computing the true posterior distribution $p(w\mid x^1,x^2,\ldots,x^N,\theta^1,\theta^2,\ldots,\theta^N)$ directly is difficult. Therefore, we evaluate how well the distribution of signs obtained by RMHNG matches that of Gibbs sampling. Appendix~\ref{apdx:bipartite} provides more details. 
\end{itemize}

{\bf{Machine Specifications}}: The experiment was conducted on a desktop PC with an Intel(R) Core(TM) i9-9900K CPU @ 3.60GHz 3.60 GHz, 32GB of RAM, and an NVIDIA GeForce RTX 2080 SUPER GPU.
\subsection{Results} {\bf {Categorization and sign sharing}}:
 Table~\ref{tab:result-synthetic} shows the ARI and $\kappa$ for each method used on the artificial data. As shown in Table~\ref{tab:result-synthetic}, the ARI values for RMHNG were consistently close to those of Gibbs sampling, with a maximum difference of only 0.1. This indicates that RMHNG had a similar category classification accuracy as Gibbs sampling.
In this setting, OS performed even better than RMHNG, achieving the highest values for both ARI and $\kappa$. This might be because OS facilitated the mixing process by introducing randomness in sampling.
On the other hand, OS\&LL and LL exhibited relatively low values for both ARI and $\kappa$. Notably, even with approximations, RMHNG had higher agent classification accuracy and sign-sharing rate than both No communication and All acceptance.

\begin{table}[bt!h]
\centering
        \caption{Experimental results for synthetic data: Each method was tested $5$ times, and for each agent, ARI and $\kappa$ were calculated when $I$ was between $91$ and $100$. Mean $\pm$ standard deviation of obtained $50 (5 \times 10)$ARI and $\kappa$ are shown. The highest scores are shown in bold, and the second-highest scores are underlined.}
        \label{tab:result-synthetic}
\scalebox{0.8}{
\begin{tabular}{cccccc}
\hline
                 & ARI         & ARI         & ARI         & ARI         &           \\
condition         & (Agent 1)   & (Agent 2)   & (Agent 3)   & (Agent 4)   & $\kappa$         \\ \hline
RMHNG            & $\underline{0.91\pm0.02}$    & $\underline{0.92\pm0.01}$   & $\underline{0.92\pm0.01}$   & $\underline{0.88\pm0.02}$   & $\underline{0.92\pm0.01}$ \\
OS        & $\mathbf{0.93\pm0.02}$   & $\mathbf{0.94\pm0.01}$   & $\mathbf{0.94\pm0.02}$   & $\mathbf{0.91\pm0.02}$   & $\mathbf{0.94\pm0.01}$ \\
LL        & $0.77\pm0.04$   & $0.8\pm0.03$    & $0.81\pm0.02$   & $0.71\pm0.03$   & $0.77\pm0.03$ \\
OS\&LL   & $0.73\pm0.09$   & $0.76\pm0.06$   & $0.77\pm0.05$   & $0.68\pm0.05$  & $0.73\pm0.07$ \\
No communication & $0.64\pm0.03$  & $0.67\pm0.03$   & $0.65\pm0.01$   & $0.60\pm0.02$   & $-0.02\pm0.14$ \\
All acceptance   & $0.008\pm0.005$ & $0.009\pm0.007$ & $0.009\pm0.007$ & $0.009\pm0.007$ & $0.42\pm0.02$ \\ \hline
Gibbs sampling   & \multicolumn{4}{c}{$0.98\pm0.01$}                        & -         \\ \hline
\end{tabular}
}
\end{table}
{\bf{Change in ARI and $\kappa$ for each iteration}}: Figure~\ref{fig:syn-flow} shows the ARI (right) and $\kappa$ (left) for each iteration ($i$ in Algorithm~\ref{alg:RMHNG}). From the left graph in Figure~\ref{fig:syn-flow}, we can see that RMHNG, OS, and LL converge faster in terms of ARI, in that order, among the RMHNG and its approximation methods. OS\&LL show an upward trend in ARI even at the 100th iteration, indicating that they have not converged. No communication has the fastest convergence in ARI among all the methods. As for All acceptance, we can see that the ARI does not show an upward trend even as the iteration count increases, compared to other methods.
From the right graph in Figure~\ref{fig:syn-flow}, we can see that RMHNG, OS, and LL converge faster in terms of $\kappa$, in that order, among the RMHNG and its approximation methods. OS\&LL show an upward trend in $\kappa$ even at the 100th iteration, indicating that they have not converged. No communication and All acceptance do not show an upward trend in $\kappa$ even as the iteration count increases, compared to other methods.

\begin{figure}[tb!p]
\centering
  \begin{minipage}[b]{0.45\columnwidth}
    \centering
    \includegraphics[width=\columnwidth]{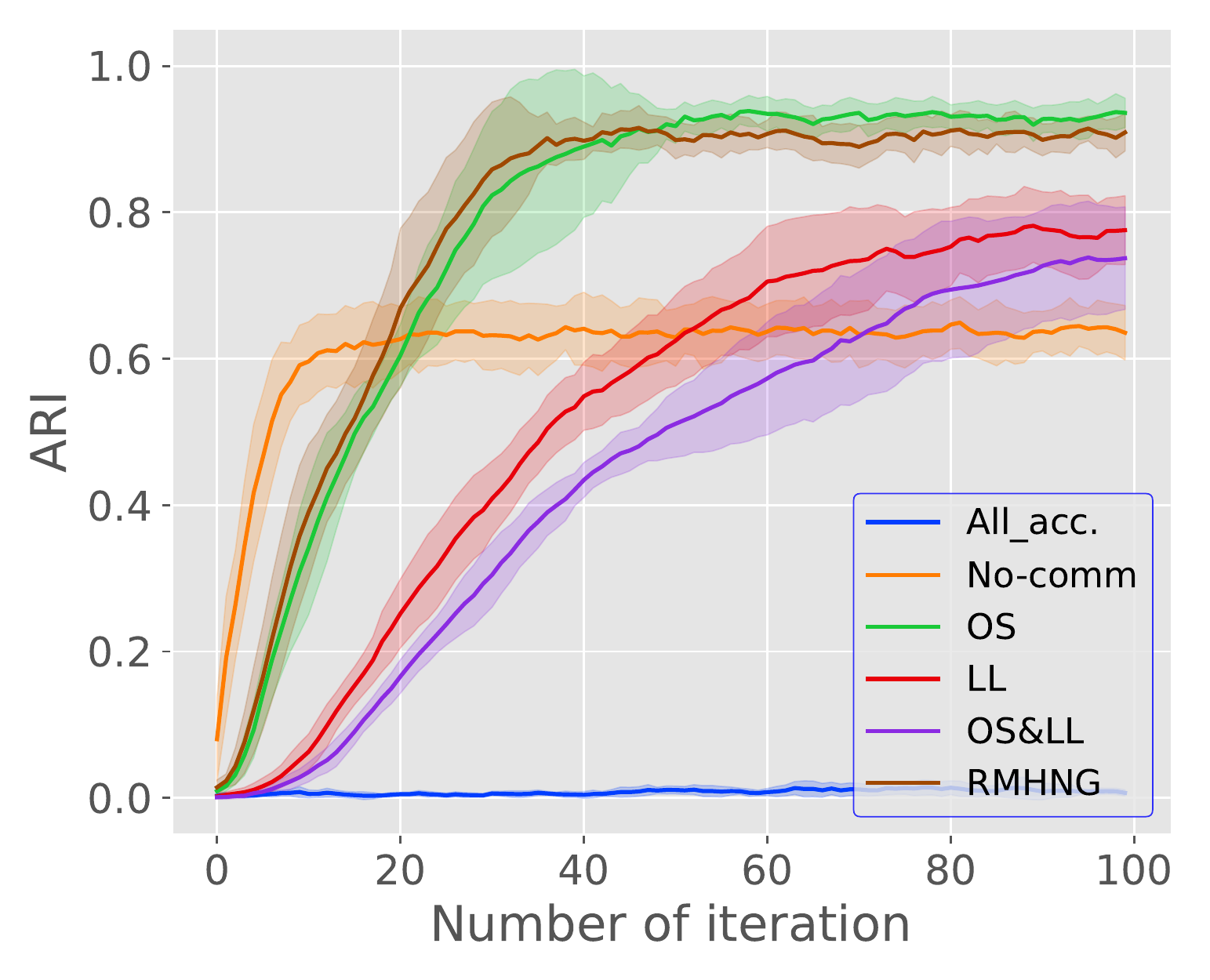}
  \end{minipage}
  \hspace{0.04\columnwidth} 
  \begin{minipage}[b]{0.45\columnwidth}
    \centering
    \includegraphics[width=\columnwidth]{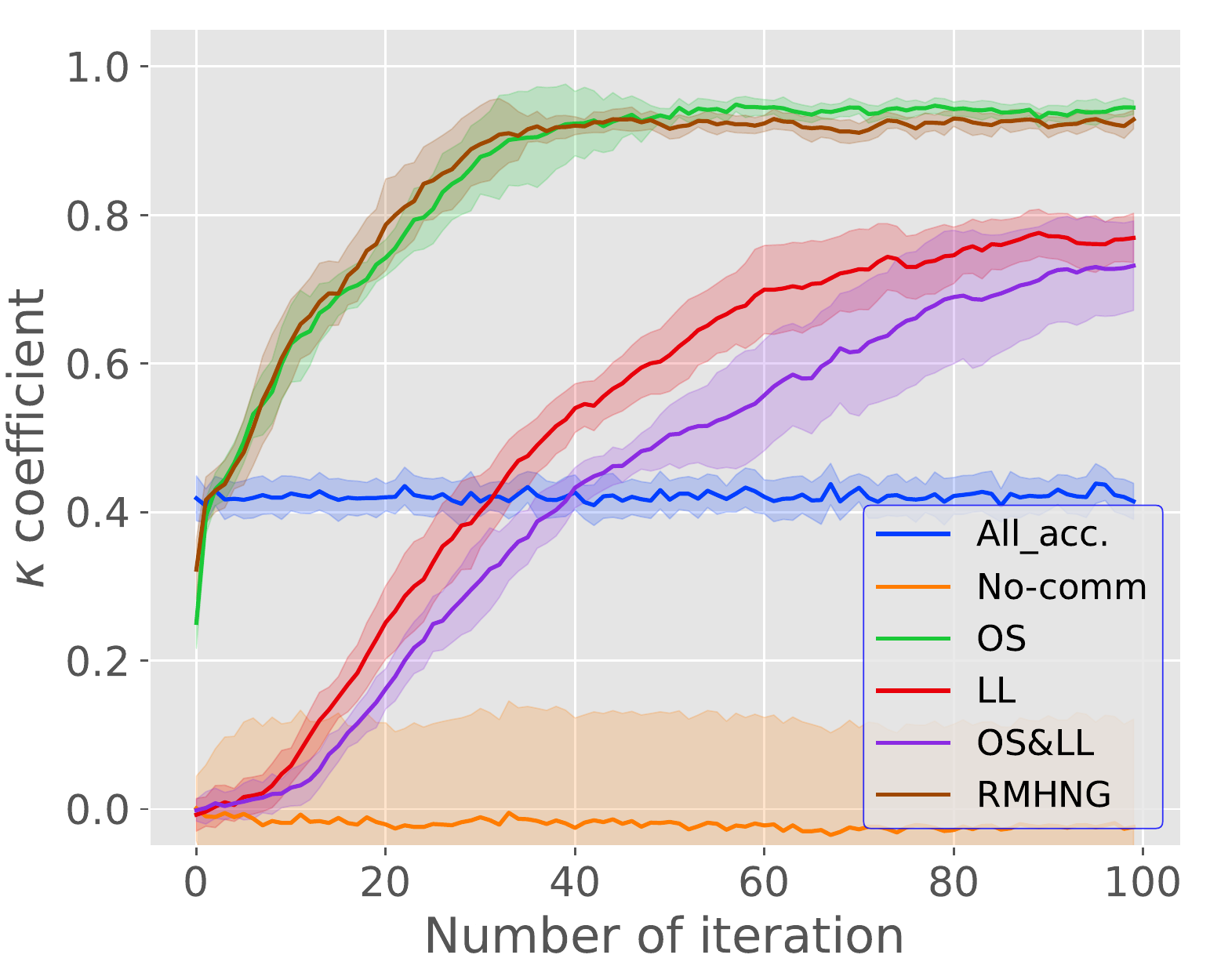}
  \end{minipage}
  \caption{ARI (left) and $\kappa$ (right) for each iteration when using artificial data}
\label{fig:syn-flow}
\end{figure}

{\bf{Computation time}}: Figure~\ref{fig:time_process} shows the average computation time for varying values of $M$ and $T$ in RMHNG. As shown in the figure, it can be seen that the computation time increases logarithmically as $T$ increases. Considering that the vertical axis is logarithmic, this confirms that the computation time follows the computational complexity of $O(T^{M-1})$. Additionally, it can be confirmed that significant reductions in computation time can be achieved by approximating RMHNG with OS ($T=1, M=3$), LL ($T=4, M=1$), or OS\&LL ($T=1, M=1$). Specifically, RMHNG ($T=4, M=3$) took 3178 seconds, OS ($T=1, M=3$) took 77 seconds, LL ($T=4, M=1$) took 187 seconds, and OS\&LL ($T=1, M=1$) took 71 seconds. 


{\bf{Decentralized posterior inference}}: Figure~\ref{fig:syn_matching} shows the results of calculating how closely the sign distribution obtained by each method matches that obtained by Gibbs sampling in the last 10 iterations (91-100 iterations) for each method. RMHNG shows a value of 0.96, indicating that the sign distribution obtained by RMHNG matches that obtained by Gibbs sampling by 96\%. This confirms that RMHNG is an approximate Bayesian estimator for the posterior distribution $p(w\mid x^1,x^2,\ldots,x^N,\theta^1,\theta^2,\ldots,\theta^N)$. Additionally, OS shows a value of 0.9 or higher, indicating that it is also an approximate Bayesian estimator for the posterior distribution $p(w\mid x^1,x^2,\ldots,x^N,\theta^1,\theta^2,\ldots,\theta^N)$. Although LL and OS\&LL have lower values compared to LL and OS, respectively, they are found to have higher matching rates with the sign distribution obtained by Gibbs sampling than No communication and All acceptance.

\begin{figure}[tb!p]
    \centering
    \scalebox{0.5}{
    \includegraphics[width=1.0\linewidth]{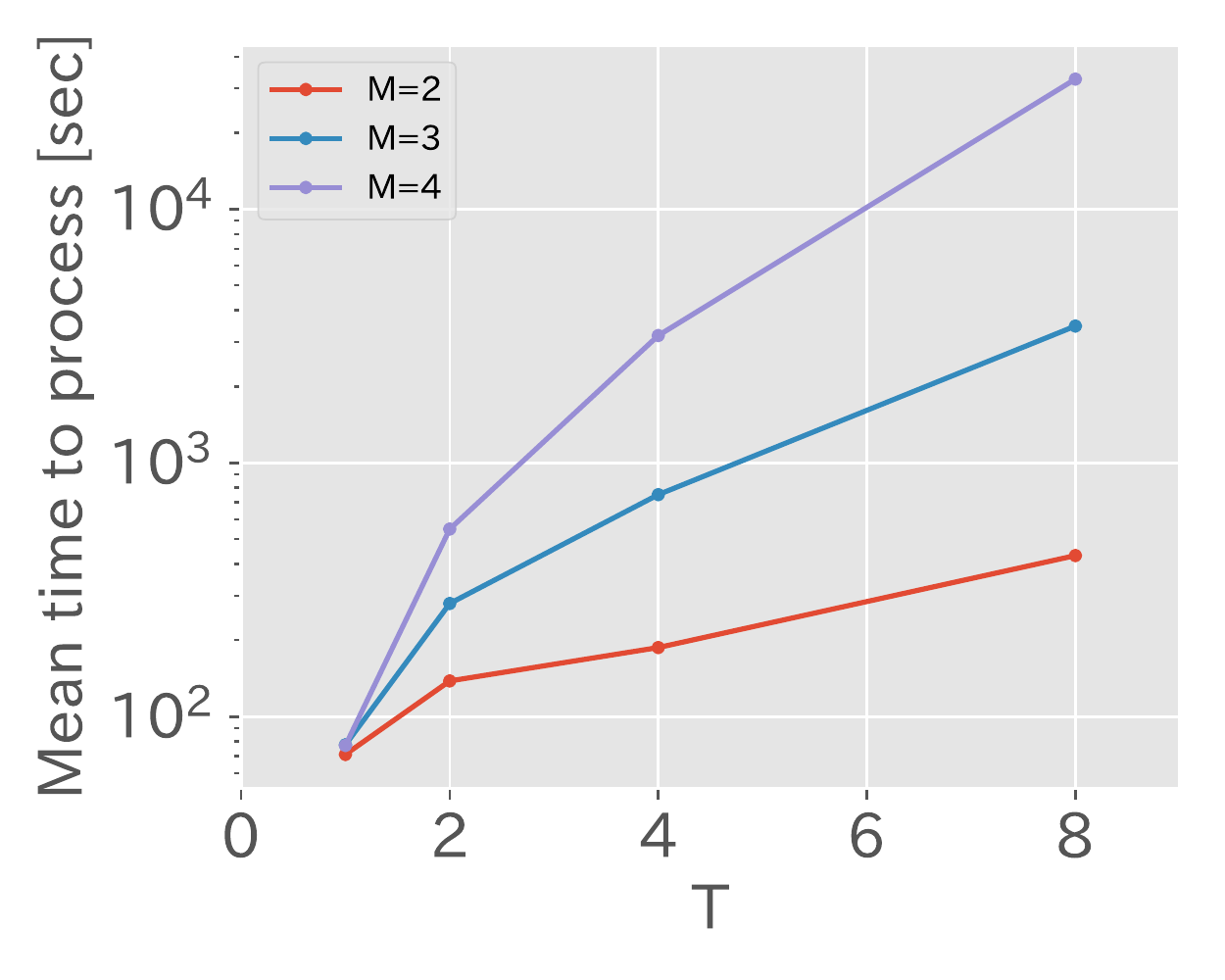}
    }
    \caption{This shows the mean of computation time when changing values of $M$ and $T$ in RMHNG. The horizontal axis represents the value of $T$, while the vertical axis represents the mean of computation time using a logarithmic scale with a base of 10. }
    \label{fig:time_process}
\end{figure}

\begin{figure}[tb!p]
    \centering
    \scalebox{0.6}{
    \includegraphics[width=1.0\linewidth]{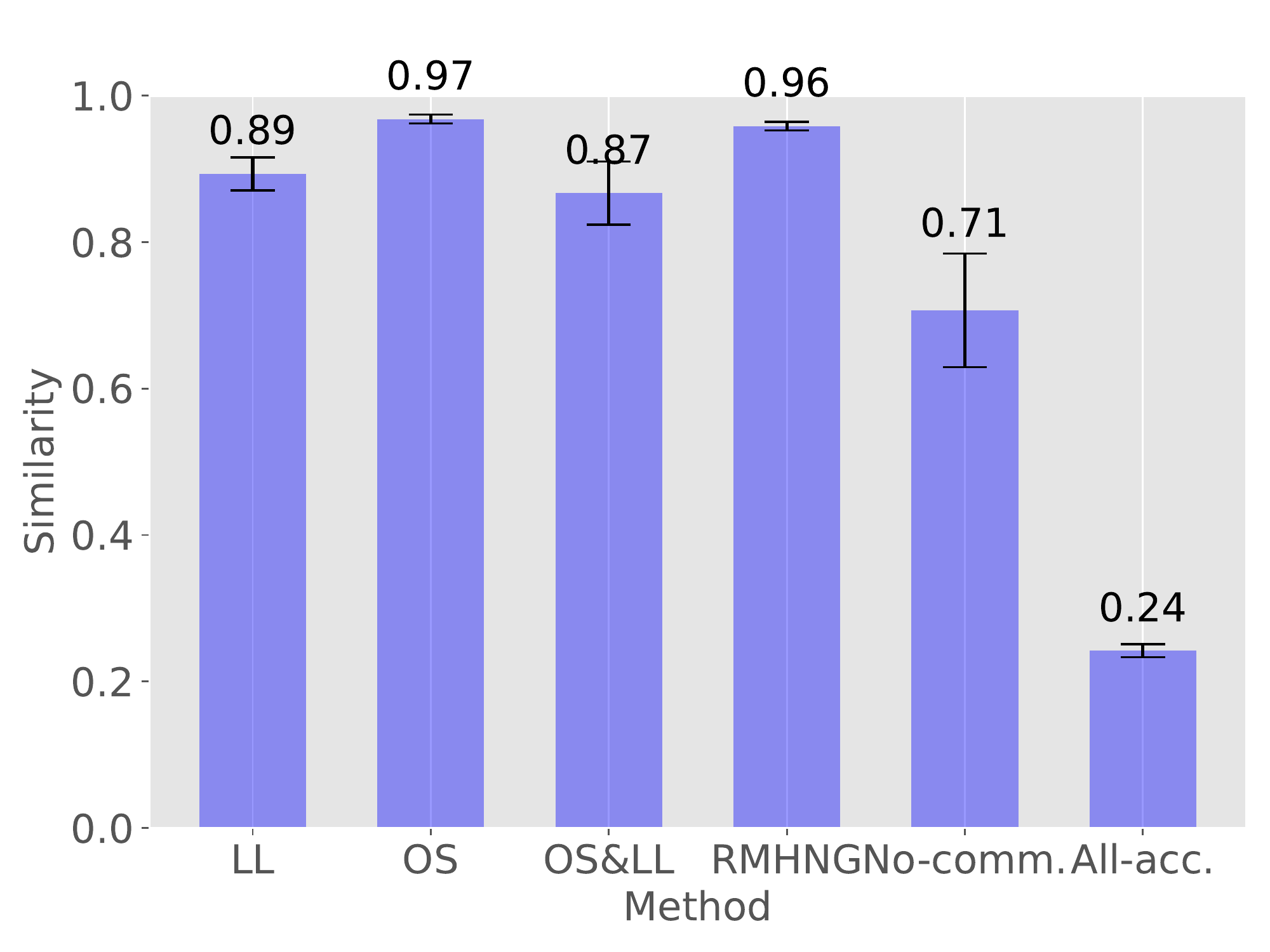}
    }
    \caption{Distribution of signs obtained by various methods and degree of agreement between the distribution of signs obtained by Gibbs sampling}
    \label{fig:syn_matching}
\end{figure}

\section{Experiment 2: YCB object dataset}\label{sec:YCB}
\subsection{Conditions}
We evaluated RMHNG using the multi-agent Inter-GMM with four agents ($N=4$) on a real image dataset. For all experiments (except for measuring computation time), the number of iterations ($I$) was set to 100, and each experiment was conducted five times.

{\bf{Dataset}}: We evaluated the performance of RMHNG using the YCB object dataset. We selected several objects from the dataset, and their names are listed in Figure~\ref{fig:overview}~(a). Figure~\ref{fig:YCB_overview}~(b), shows an overview of the dataset, where we divided the images of each object into four sets and assigned each set to one of the four different agents. Each set consisted of 30 images.
Specifically, images ranging from 0° to 87° were assigned to agent 1, those from 90° to 177° were assigned to agent 2, those from 180° to 267° were assigned to agent 3, and those from 270° to 267° were assigned to agent 4. 

{\bf{Feature extraction}}: Firstly, we cropped the original images from $4272 \times 2848$ to a size of $2000 \times 2000$ from the center. Next, we reduced the cropped images to a size of $300 \times 300$ to prevent any degradation in image quality. Finally, we cropped the images further to a size of $224 \times 224$ from the center. We used the resulting images as the observations for each agent, denoted as $o^n_d$~\footnote{We conducted two rounds of cropping for two reasons. Firstly, direct reduction of an ultra-high-resolution image can lead to degradation of image quality, so it was necessary to crop the image to a certain size before reducing it. Secondly, some objects were almost entirely obscured or markers were largely reflected, making it difficult to crop directly from the original image.}.

Feature extraction was performed using SimSiam~\citep{chen2021exploring}, a representation learning method based on self-supervised learning, pre-trained on the collected cropped YCB-object dataset. The feature extractor outputted $512$-dimensional vectors. To address the issue of high feature dimensionality compared to the small amount of data available, principal component analysis (PCA) was used to reduce the features to $10$ dimensions\footnote{This was conducted to avoid zero variance in dimensions that cannot be inferred, preventing divergence of the accuracy matrix of GMMs.}. Figure~\ref{fig:pca10} shows a visualization of the features of all data and the features observed by each agent using PCA. From this figure, it can be expected that some degree of categorization is possible.

{\bf{Hyperparameters}}:
 All agents were assigned the same hyperparameters, with values set as follows: $\beta=1$, $m=0$, $W=100 \times \mathbf{I}$, and $\nu=1$, where $\mathbf{I}$ is a $10$-dimensional identity matrix $\mathbf{I}$.

{\bf Compared method} and {\bf evaluation criteria} are the same as those in  the Experiment 1.

\begin{figure}[tb!p]
    \centering
    \includegraphics[width=1.0\linewidth]{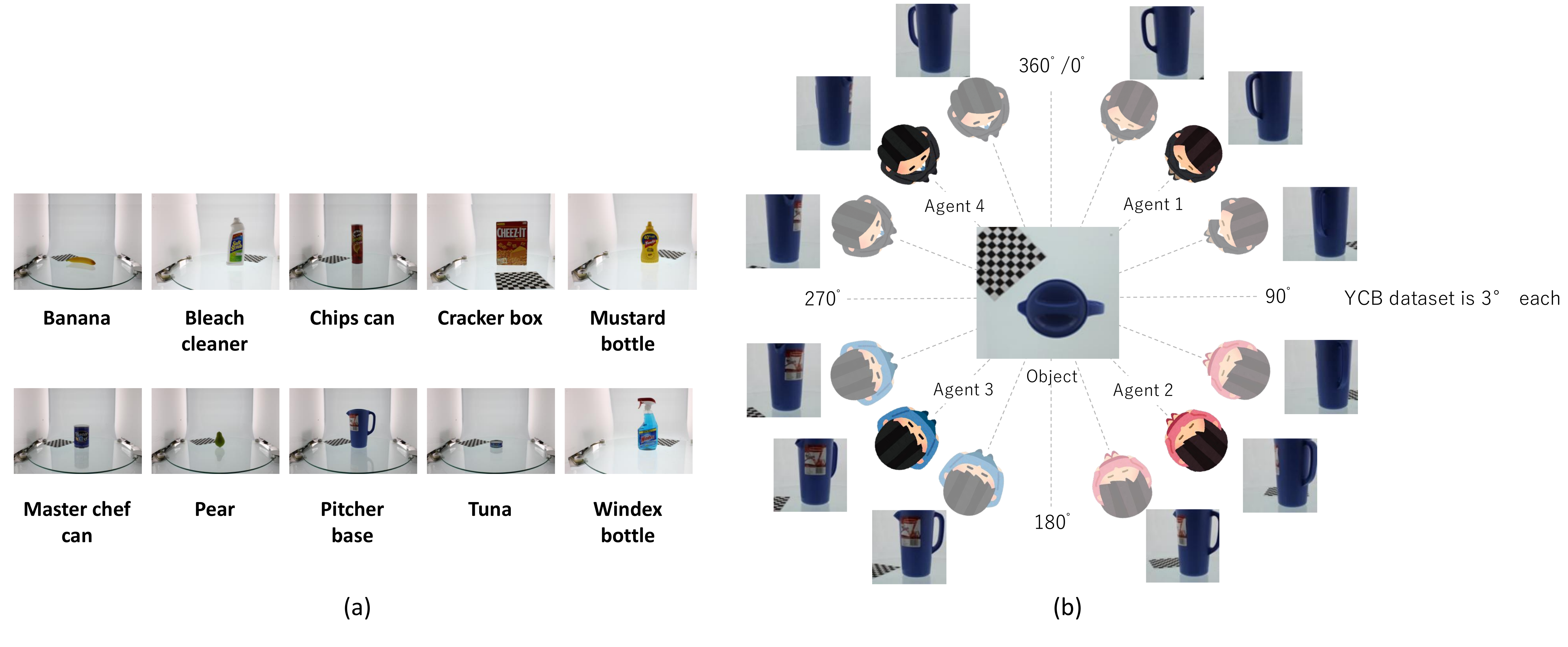}
    \caption{(a): Type of YCB object dataset utilized in the experimental analysis (b): Partition diagram of YCB object dataset. we divided the images of each object into four sets and assigned each set to one of the four different agents. Each set consisted of 30 images.
Specifically, images ranging from 0° to 87° were assigned to agent 1, those from 90° to 177° were assigned to agent 2, those from 180° to 267° were assigned to agent 3, and those from 270° to 267° were assigned to agent 4. }
    \label{fig:YCB_overview}
\end{figure}

\begin{figure}[tb!p]
    \centering
    \includegraphics[width=1.0\linewidth]{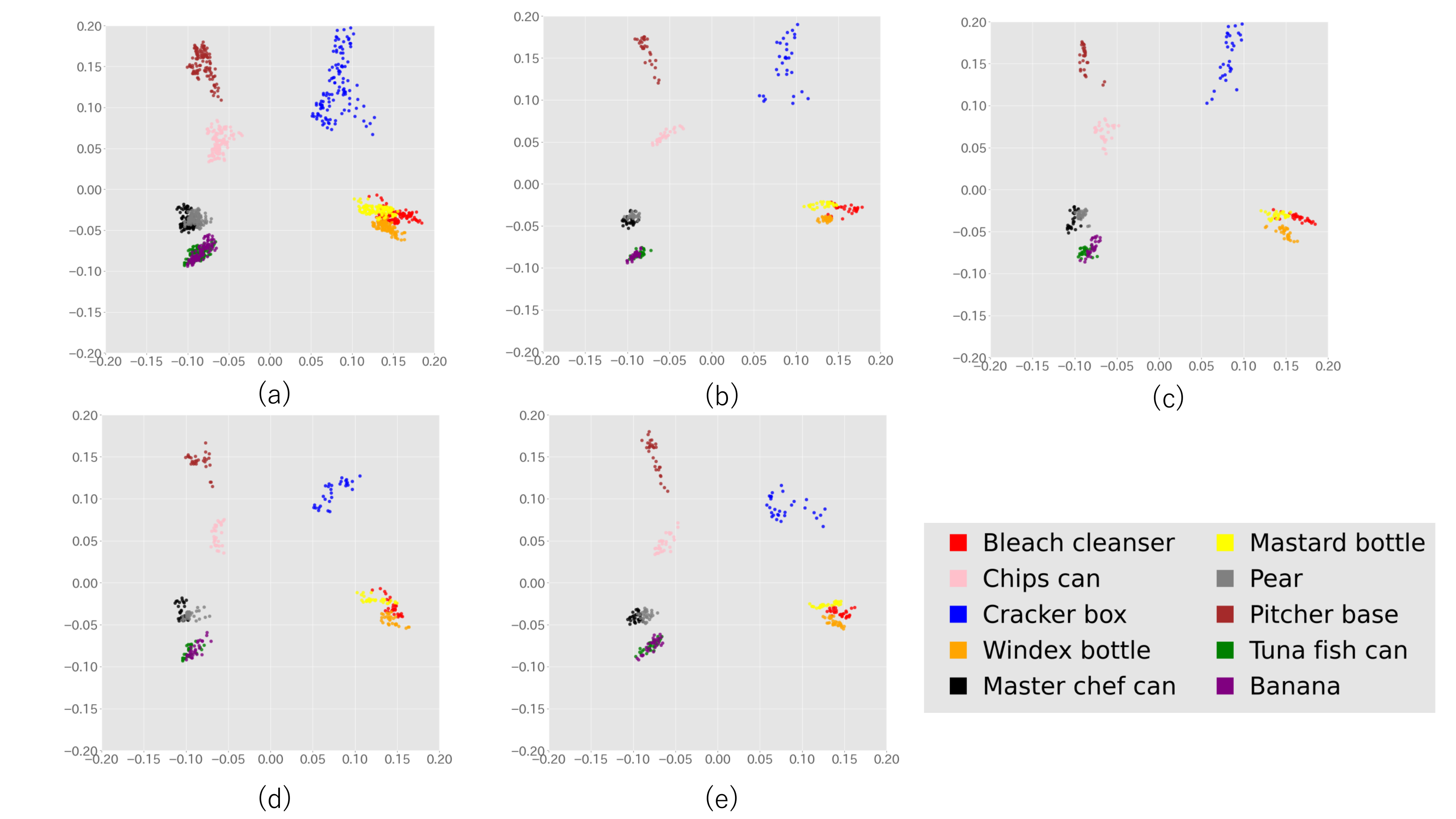}
    \caption{Features of the entire dataset and the features of individual agents' observations are visualized by 2D-PCA (a): Features of all data visualized by 2D-PCA (b): PCA visualization of Agent 2's observations (c): PCA visualization of Agent 3's observations (D): PCA visualization of Agent 4's observations}
    \label{fig:pca10}
\end{figure}
  
\subsection{Result}
{\bf {Categorization and sharing signs}}: Table~\ref{tab:result-ycb} shows the ARI and $\kappa$ for each method on the YCB object dataset. It is observed that RMHNG and Gibbs sampling have similar category classification accuracy with a maximum difference of only 0.04. Among the RMHNG approximations, OS had the highest ARI and $\kappa$ values. Interestingly, it showed a value close to that of RMHNG for the $\kappa$. OS\&LL had the lowest values for both ARI and $\kappa$. However, the difference in ARI between LL, OS, and OS\&LL was at most 0.02, indicating similar performance. In the YCB object dataset experiments, although OS showed higher ARI than RMHNG in the synthetic data experiment, for Agent1, OS showed lower ARI than RMHNG, while for other agents, it showed similar values. Compared to OS\&LL, No communication showed equivalent ARI for Agent1, lower ARI for Agent2, and higher ARI for other agents. However, the $\kappa$ was the lowest among all methods for No communication. All acceptance had the lowest ARI among all methods and the highest $\kappa$ among all methods.

\begin{table}[bt!p]
\centering
        \caption{Experimental results for YCB object dataset: Each method was tested five times, and for each agent, the ARI and $\kappa$ were calculated when $I$ was $91-100$. The mean $\pm$ standard deviation of obtained $50 (5 \times 10)$ARI and $\kappa$ are shown. Highest scores are shown in bold, and second-highest scores are underlined.}
        \label{tab:result-ycb}
\scalebox{0.8}{
\begin{tabular}{cccccc}
\hline
                 & ARI         & ARI         & ARI         & ARI         &           \\
conditon         & (Agent 1)   & (Agent 2)   & (Agent 3)   & (Agent 4)   & $\kappa$         \\ \hline
RMHNG            & $\mathbf{0.61\pm0.05}$    & $\mathbf{0.59\pm0.05}$   & $\mathbf{0.59\pm0.05}$   & $\mathbf{0.59\pm0.05}$   & $\underline{0.99\pm0.04}$ \\
OS        & $\underline{0.59\pm0.06}$   & $\mathbf{0.59\pm0.06}$   & $\mathbf{0.59\pm0.08}$   & $\mathbf{0.59\pm0.08}$   & ${0.98\pm0.05}$ \\
LL        & ${0.56\pm0.11}$   & $\underline{0.57\pm0.09}$    & $\underline{0.56\pm0.11}$   & $\underline{0.56\pm0.11}$   & $0.96\pm0.06$ \\
OS\&LL   & $0.55\pm0.07$   & ${0.55\pm0.07}$   & ${0.54\pm0.06}$   & ${0.54\pm0.06}$  & ${0.95\pm0.09}$ \\
No communication & $0.55\pm0.1$  & $0.5\pm0.1$   & ${0.55\pm0.07}$   & ${0.55\pm0.08}$   & $-0.03\pm0.08$ \\
All acceptance   & ${0.47\pm0.08}$ & ${0.47\pm0.08}$ & ${0.47\pm0.08}$ & ${0.47\pm0.08}$ & $\mathbf{1.0\pm0.0}$ \\ \hline
Gibbs sampling   & \multicolumn{4}{c}{$0.63 \pm 0.05$}                        & -         \\ \hline
\end{tabular}
}
\end{table}

{\bf{Change in ARI and $\kappa$ for each iteration}}: Figure~\ref{fig:fycb-flow} shows the ARI (right) and $\kappa$ (left) for each iteration ($i$) in Algorithm~\ref{alg:RMHNG} for various methods, while Figure~\ref{fig:syn-flow} shows the convergence of the $\kappa$ for synthetic data.
From the left figure in Figure~\ref{fig:fycb-flow}, we can see that the RMHNG method has the fastest convergence of ARI, followed by OS, OS\&LL, and LL. Regarding OS\&LL, we can see that ARI did not converge when using synthetic data, but it did converge when using the YCB object dataset. No communication had the fastest convergence of ARI among all the methods. As for All acceptance, we can see that ARI did not show an increasing trend with iteration in synthetic data, but it did show an increasing trend when using the YCB object dataset.
From the right figure in Figure~\ref{fig:syn-flow}, we can see that the RMHNG method had the fastest convergence of the $\kappa$, followed by OS, LL, and OS\&LL. No communication did not show any increasing trend compared to other methods. As for All acceptance, we can see that the $\kappa$ did not show an increasing trend with iteration when using synthetic data, but it did show an increasing trend when using the YCB object dataset.

{\bf{Decentralized posterior inference}}: Figure~\ref{fig:ycb_matching} shows the results of calculating the degree of similarity between the distribution of the sign obtained by each method and that obtained by Gibbs sampling in the last 10 iterations (91-100 iterations) for each method. RMHNG showed a value of 0.76, indicating that the distribution of the sign obtained by RMHNG matched that obtained by Gibbs sampling by 76\%. Among the methods that approximated RMHNG, OS showed the highest value, both in the synthetic data experiment and the YCB object dataset experiment. Additionally, all approximation methods showed higher values than No communication.
\begin{figure}[tb!p]
\centering
  \begin{minipage}[b]{0.45\columnwidth}
    \centering
    \includegraphics[width=\columnwidth]{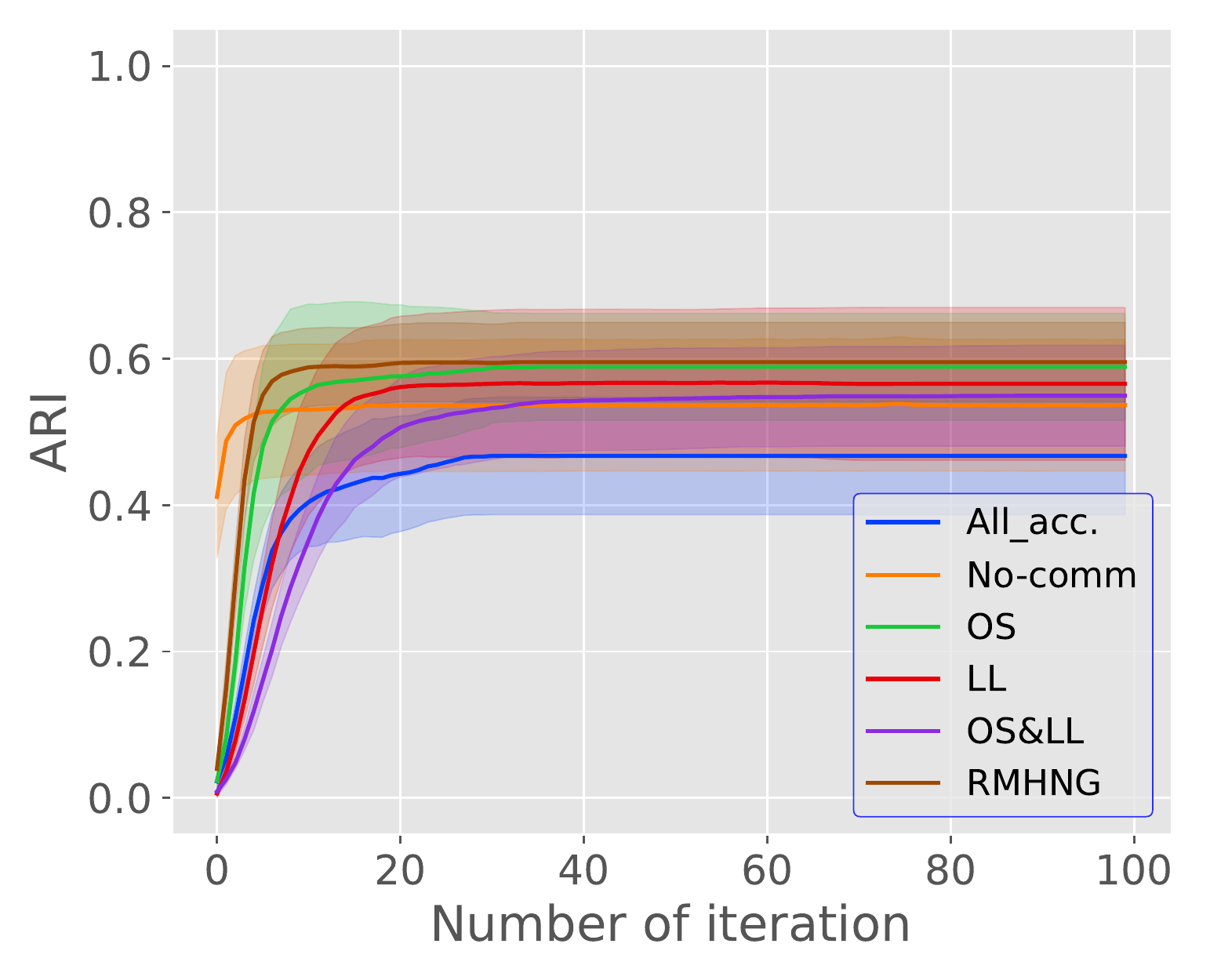}
  \end{minipage}
  \hspace{0.04\columnwidth} 
  \begin{minipage}[b]{0.45\columnwidth}
    \centering
    \includegraphics[width=\columnwidth]{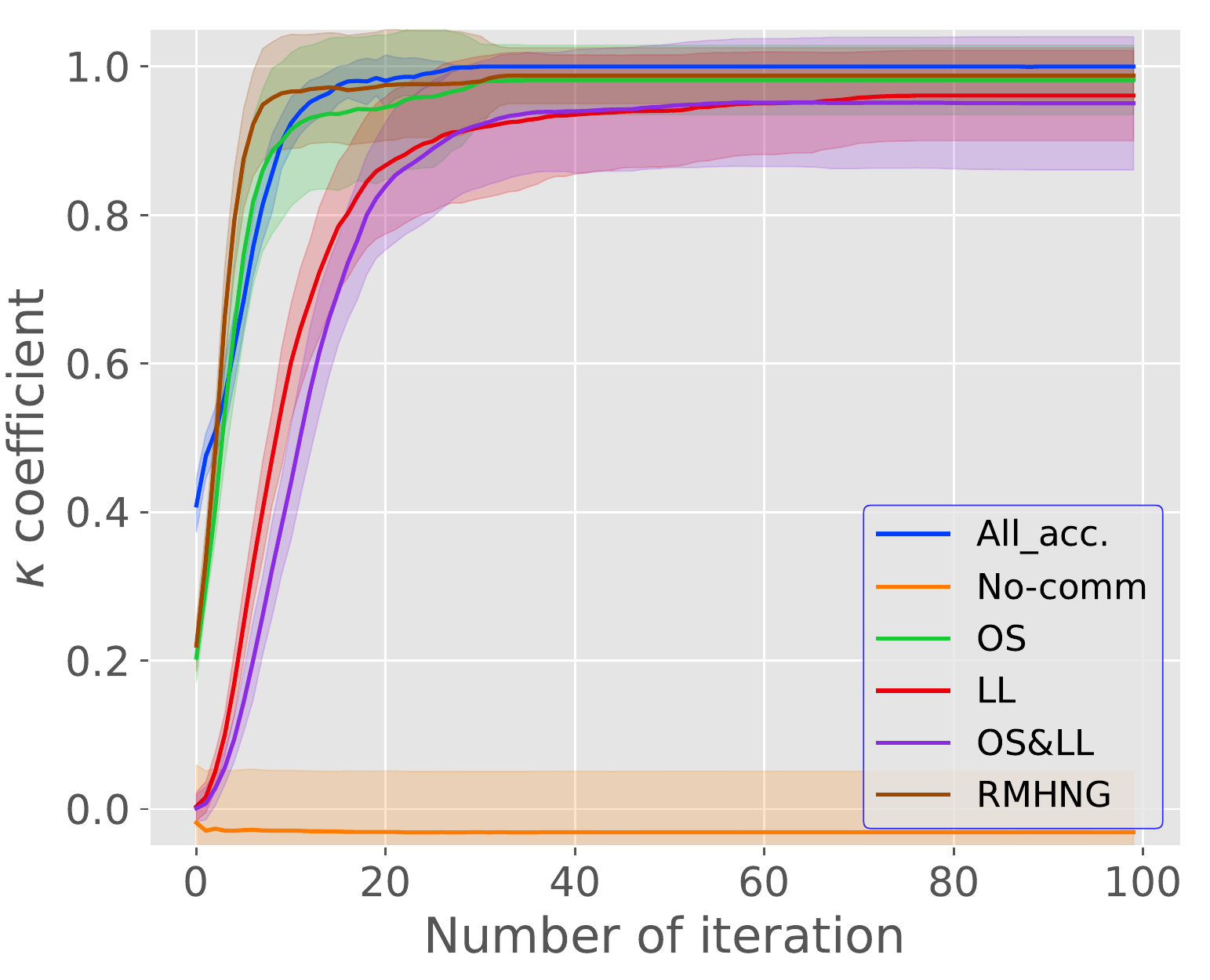}
  \end{minipage}
  \caption{ARI (left) and $\kappa$ (right) for each iteration when using YCB object dataset}
\label{fig:fycb-flow}
\end{figure}

\begin{figure}[tb!p]
    \centering
    \scalebox{0.7}{
    \includegraphics[width=1.0\linewidth]{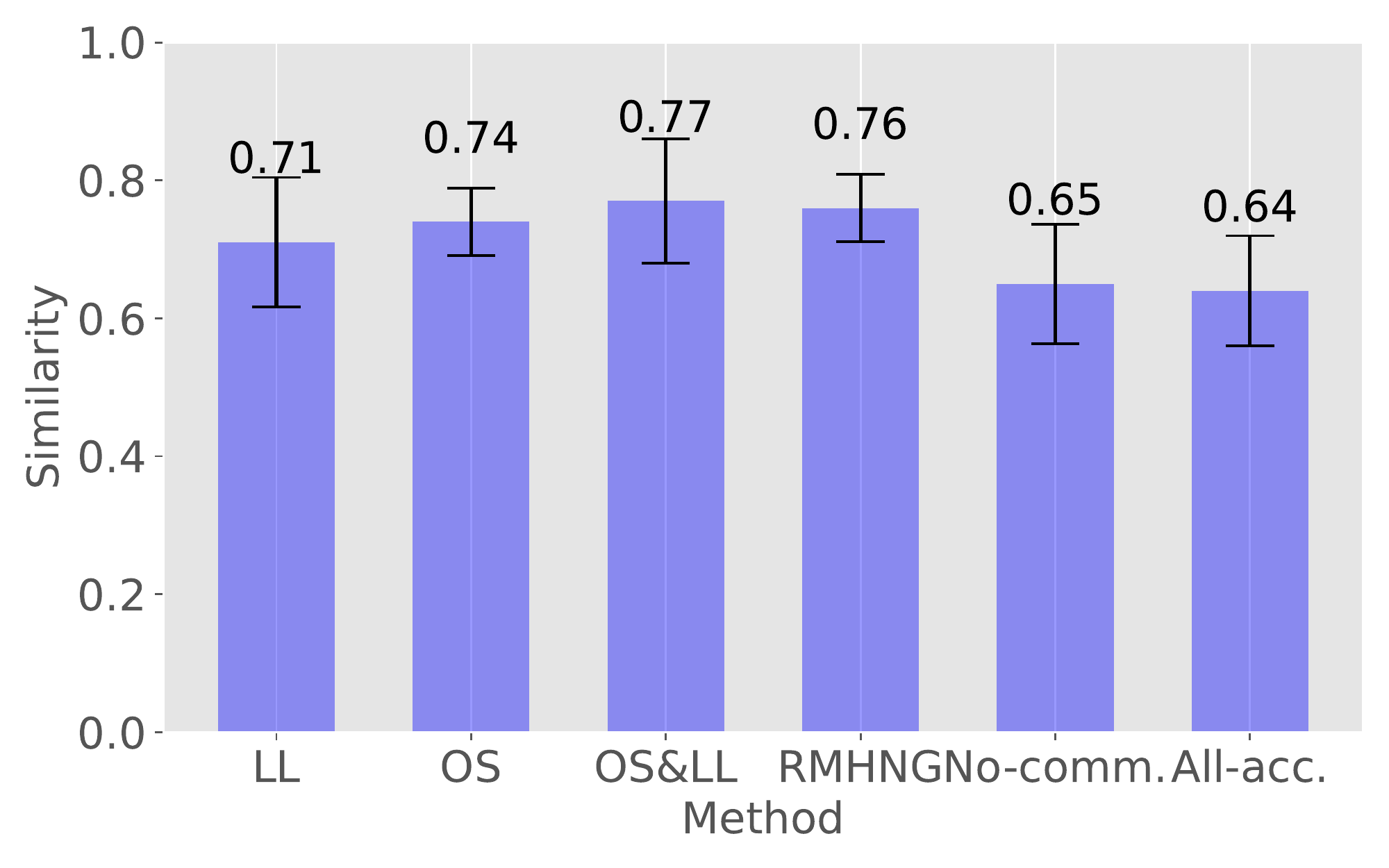}
    }
    \caption{Percentage agreement with Gibbs sampling.}
    \label{fig:ycb_matching}
\end{figure}


When comparing Table~\ref{tab:result-synthetic} and Table~\ref{tab:result-ycb}, we observed that the ARI values for RMHNG, OS, LL, and OS\&LL were higher for the synthetic dataset, while the $\kappa$ were higher for the YCB object dataset. This may be because the YCB object dataset is easier to categorize when observed partially by an agent, but some objects are so similar that agents naturally regard them as instances of a single category. For example, in Figure~\ref{fig:pca10}, the mustard bottle, bleach cleanser, and Windex bottle have similar feature distributions, making it difficult to cluster them according to the ground truth. Therefore, using the YCB object dataset leads to a decrease in ARI. As a result, when using the YCB object dataset, there was little difference in ARI for No communication and methods based on RMHNG compared to when using the synthetic dataset.

\section{Conclusion}\label{sec:conclusion}
In this study, we extended the MHNG to the $N$-agent scenario by introducing the RMHNG, which serves as an approximate decentralized Bayesian inference method for the posterior distribution shared by agents, similar to the MHNG. We demonstrated the effectiveness of RMHNG in enabling multiple agents to form and share a symbol system using synthetic and real image data. To address computational complexity, we proposed two types of approximations: OS and LL approximations. Evaluation metrics, such as the ARI and the $\kappa$, were used to assess the performance of communication in each iteration of the naming game. Results showed that the $4$-agent naming game successfully facilitated the formation of categories and effective sign-sharing among agents.
Moreover, the approximated RMHNG exhibited higher ARI and $\kappa$ compared to the No communication condition, showing that the approximate version of RMHNG could perform symbol emergence in a population. Additionally, we assessed the agreement between the sign distributions obtained by RMHNG and Gibbs sampling, confirming that RMHNG approximates the posterior distribution with a degree of agreement exceeding $87\%$ for the synthetic data and $71\%$ for the YCB object data. This result demonstrates that RMHNG could successfully approximate the posterior distribution over signs given every agent's observations.
Several future perspectives emerge from this study. Firstly, we plan to analyze the behavior of the RMHNG in populations with a larger number of agents. Although we focused on the $4$-agent scenario due to the computational cost of the original RMHNG ($O(IDT^{(N-1)})$), we empirically observed that the OS approximation performed well in many cases. Unlike the original RMHNG, the OS-approximated version exhibits scalability in terms of the number of agents ($O(N)$), enabling simulations with larger populations. This scalability opens up possibilities for providing valuable insights into language evolution through the MHNG framework.
Additionally, extending the categorical signs to more complex signs, such as sequences of words, represents a natural progression for our research. Investigating the dynamics of communication with more intricate sign systems will shed light on the evolution and complexity of language.

\bibliographystyle{Frontiers-Harvard} 
\bibliography{test}

\section*{Acknowledgments}
This study was partially supported by the Japan Society for the Promotion of Science (JSPS) KAKENHI under Grant JP21H04904 and JP18K18134 and by MEXT Grant-in-Aid for Scientific Research on Innovative Areas 4903 (Co-creative Language Evolution), 17H06383.

\appendix
\section*{Appendix}
\section{Evaluation of the decentralized posterior inference architecture}\label{apdx:bipartite}

In order to evaluate the efficacy of RMHNG as an approximate Bayesian estimator for the posterior distribution $p(w\mid x^1,x^2,\ldots,x^N,\theta^1,\theta^2,\ldots,\theta^N)$, a comparison was made between RMHNG and the actual posterior distribution. However, direct computation of the true posterior distribution $p(w\mid x^1,x^2,\ldots,x^N,\theta^1,\theta^2,\ldots,\theta^N)$ presented significant challenges. Therefore, the evaluation focused on the degree of concurrence between the sign distribution generated by RMHNG and that produced by Gibbs sampling.
The evaluation was conducted for the last 10 iterations (i.e., $91 - 100$ iterations) of RMHNG and Gibbs sampling. Let $f_{d,w}^R$ and $f_{d,w}^G$ be the number of times the word $w$ was sampled using RMHNG and Gibbs sampling, respectively, for the $d$-th dataset. The similarity between the two methods was calculated as $\sum_{d=1}^{D}\sum_{k=1}^{K}\min(f_{d,k}^R,f_{d,k}^G)$.
However, due to the singularity of the GMM, label switching (i.e., swapping of signs) between different inference results needed to be addressed. To solve this problem, bipartite graph matching was performed to correspond a clustering result with another.
To perform bipartite graph matching, the sign obtained by RMHNG was considered as the point set $V^R=\{v^R_0,v^R_1,\ldots,v^R_K\}$, and the sign obtained by Gibbs sampling was considered as the point set $V^G=\{v^G_0,v^G_1,\ldots,v^G_K\}$. The edge connecting $V^R_i$ and $V^G_j$ was denoted by $E_{i,j}$, and the set of all edges was denoted by $E=\{e_{0,0},e_{0,1},\ldots,e_{i,j},\ldots,e_{K,K}\}$. The graph $G=(V^G \lor V^R,E)$ was a complete bipartite graph. If the gain of each pair $(v^R_a,v^G_b)$ was $\sum_{d=1}^{D}\min(f_{d,a}^R,f_{d,b}^G)$, then the sign replacement problem could be reduced to a weighted maximum bipartite matching problem.
To simplify the problem further, the gain of each pair was multiplied by $(-1)\times \sum_{d=1}^{D}\min(f_{d,a}^R,f_{d,b}^G)$. This reduced the weighted maximum two-part matching problem to a minimum cost flow problem, which could be solved using the Hungarian method. Finally, the similarity was calculated by $\frac{1}{10D}\sum_{d=1}^{D}\sum_{k=1}^{K}\min(f_{d,k}^R,f_{d,k}^G)$, where $\frac{1}{10D}$ was a normalization factor.

\section{ARI and $\kappa$}\label{apdx:ari}
ARI is a widely used measure for evaluating clustering performance by comparing the clustering results with the ground-truth labels. Unlike precision, which is calculated by directly comparing estimated labels to ground-truth labels and often used in the evaluation of classification systems trained using supervised learning, ARI considers label-switching effects in clustering. The formula for ARI is given by Equation (\ref{eq:ari}), where RI represents the Rand Index. Further details can be found in \citet{hubert1985comparing}.
\begin{align}
{\rm ARI} = \frac{\rm{RI - Expected\ RI}}{{\rm Max\ RI - Expected\ RI}}
\label{eq:ari}
\end{align}

The kappa coefficient ($\kappa$) is defined by Equation (\ref{eq:kappa}): \begin{align} \kappa = \frac{C_o - C_e}{1-C_e} \label{eq:kappa}. \end{align} Here, $C_o$ represents the degree of agreement of signs among agents, and $C_e$ denotes the expected value of coincidental sign agreement. The interpretation of $\kappa$ is as follows~\citep{landis1977measurement}: 
\begin{enumerate}
\item Almost perfect agreement: ($1.0 \geq \kappa > 0.80$)
\item Substantial agreement: ($0.80 \geq \kappa > 0.60$)
\item Moderate agreement: ($0.60 \geq \kappa > 0.40$)
\item Fair agreement: ($0.40 \geq \kappa > 0.20$)
\item Slight agreement: ($0.20 \geq \kappa > 0.00$)
\item No agreement: ($0.0 \geq \kappa$)
\end{enumerate}

\section{Feature extraction by SimSiam}
As a feature extractor, we utilized SimSiam~\citep{chen2021exploring}, a self-supervised representation learning technique that was pre-trained on the YCB object dataset. We followed the same network architecture and hyperparameters as outlined in the original paper~\citep{chen2021exploring}, but with a few minor adjustments.
For data augmentation, we used the following parameters, using PyTorch notation~\footnote{PyTorch, Transforming and augmenting images: \url{https://pytorch.org/vision/stable/transforms.html} }.
\textsc{RandomResizedCrop} with a scale in the range of $[0.1, 0.6]$, and \textsc{RandomGrayscale} with a probability of $0.2$. We normalized tensor images using the $\textsc{Normalize}$ function with a mean of $(0.485, 0.456, 0.406)$ and standard deviation of $(0.229, 0.224, 0.225)$
We used ResNet-18 as the \textsc{backbone}~\citep{He_2016_CVPR} and set the dimension of the output feature vector to $512$. A two-layer fully connected layer was employed as the \textsc{projector}, using an intermediate layer with a dimension of $512$. The predictor also utilized a two-layer fully connected layer, with an intermediate layer dimension of $128$. During the training phase, we set the learning rate to $0.1$ and utilized stochastic gradient descent as the optimizer. The batch size was set to 64, and we trained the network for $100$ epochs.
\end{document}